\documentclass[]{interact}

\usepackage{epstopdf}% To incorporate .eps illustrations using PDFLaTeX, etc.
\usepackage[caption=false]{subfig}% Support for small, `sub' figures and tables
\usepackage{hyperref}
\theoremstyle{plain}% Theorem-like structures provided by amsthm.sty
\newtheorem{theorem}{Theorem}[section]

\theoremstyle{definition}

\theoremstyle{remark}

\usepackage{blindtext}
\usepackage{algpseudocode}
\usepackage{natbib}
\usepackage{tikz}
\usepackage{subcaption}
\usepackage{graphicx}
\usepackage{amsmath}
\usepackage{footnote}
\usepackage[linesnumbered,ruled,vlined]{algorithm2e}
\usepackage{algorithmicx}
\usepackage{algcompatible}
\usepackage{physics}

\DeclareMathOperator*{\argmin}{arg\,min}

\usepackage{lastpage}
\begin{document}

\setcitestyle{square,numbers}

\title{{An efficient, accurate, {and interpretable} machine learning method for computing probability of failure}}

\author{
	\name{J.~Zhu\textsuperscript{a} and D.~Estep\textsuperscript{b}\thanks{CONTACT D.~Estep. Email: destep@sfu.ca}}
	\affil{\textsuperscript{a} Department of Statistics and Actuarial Science, Simon Fraser University, Burnaby, British Columbia, Canada; \textsuperscript{b}Department of Statistics and Actuarial Science, Simon Fraser University, Burnaby, British Columbia, Canada}
}

\maketitle

\begin{abstract}
{We introduce a novel machine learning method called the \emph{Penalized Profile Support Vector Machine based on the Gabriel edited set} for the computation of the probability of failure for a complex system as determined by a threshold condition on a computer model of system behavior. The method  is designed to minimize the number of evaluations of the computer model while preserving the geometry  of the decision boundary that determines the probability. It employs an adaptive sampling strategy designed to strategically allocate points near the boundary determining failure and builds a locally linear surrogate boundary that remains consistent with its geometry by strategic clustering of training points. We prove two convergence results and  we compare the performance of the method against a number of state of the art classification methods on four test problems.  We also apply the method to determine the probability of survival using the Lotka--Volterra model for competing species.}
\end{abstract}


\section{Introduction}

Computing the \textit{probability of failure}, or the probability that a system surpasses a critical threshold, is a fundamental problem  in engineering and  science, especially in applications in which safety and reliability is important (\cite{Freud1956,Santo1977,Miller1992,Au2001,Gifford2005,Silva2008,Li2010, Morio2011,POF-darts,Wang2022}). We consider a  physical or engineering system whose behavior varies according to a set of physical characteristics, e.g.,   mechanical, thermal, optical, electrical, or magnetic properties, that take values in a specified set $\Lambda\subset \mathbb{R}^d$, $d\geq 1$. We assume that the state of the system $z$ is described as the solution of a \textit{process model} $M(z,x)=0$, e.g., a differential equation encoding physical laws and conservation principles. The solution $z=z(x)$ is an implicit function of the characteristics $x \in \Lambda$. We assume that an important measurement of system response is determined by a functional $q$ applied to the solution $z$, yielding the composition $Q(x) = q(z(x))$ for $x \in \Lambda$. We define the range of $\mathcal{D} = Q(\Lambda)\subset\mathbb{R}$. {In practice, the process model is implemented as a computer model and we use $Q$ to refer to both the process model and its implementation.} 

We classify points in $\Lambda$ according to a partition into \textit{failure}  and \textit{success} events,
\begin{equation}\label{Qcritcond}
\Lambda = \Lambda_{\mathcal{F}} \cup \Lambda_{\mathcal{S}} = \{x \in \Lambda : \, Q(x)\geq q_0\}\, \bigcup \, \{x \in \Lambda : \, Q(x)< q_0\},
\end{equation}
for a fixed \textit{critical value} $q_0\in\mathcal{D}$. Assuming $\Lambda$ is bounded and equipping it with the Borel $\sigma$--algebra  $\mathcal{B}_\Lambda$ and the uniform probability measure $P$, the problem is to estimate the \textit{probability of failure}, 
{\begin{equation}\label{ProfFail}
P(\Lambda_{\mathcal{F}})=P\big(\{x \in \Lambda: Q(x) \geq q_0\}\big),
\end{equation}}
or, noting  $P(\Lambda_{\mathcal{S}})=1-P(\Lambda_{\mathcal{F}})$, the \textit{probability of success} $P(\Lambda_{\mathcal{S}})$.

The critical value $q_0$ determines a \textit{nonlinear decision boundary},
\[
Q^{-1}(q_0) = \{ x: Q(x) = q_0 \}\subset \Lambda.
\] 
Under generic conditions on $M$ and $q$, $Q^{-1}(q)$ is a compact piecewise smooth manifold in $\Lambda$ for each $q \in \mathcal{D}$ and the manifolds corresponding to different output values do not intersect (\cite{BET+14,LYang,Shi25}).  Given a critical value $q_0\in\mathcal{D}$, the points in $\Lambda$ are classified according to which ``side'' of $Q^{-1}(q_0)$ they lie, see Fig.~\ref{fig:Brusscontour}.  {Therefore, computation of probability of failure can be framed as a classification problem  in which a model of the physics determines the {nonlinear decision boundary}.}

To illustrate, we consider the Brusselator model of an autocatalytic chemical reaction in which one of the reaction products is a catalyst for the reaction (\cite{Belousov,Brusselator}).  Such reactions can exhibit complex dynamics such as self-organizing spatial patterns and oscillatory behavior.  The model considers two chemical species reacting in a homogeneous medium in which the species diffuse freely. The model is a set of differential equations whose solution is determined by two rate parameters and the initial conditions. In dimensionless form, the model is
\begin{equation}\label{Brussmodel}
	M(z,x) = \begin{cases} 
		\frac{d z_1}{d t} - x_1+(1+x_2) z_1 -z_1^2z_2,\\
		\frac{d z_2}{d t} - x_2z_1 + z_1^2z_2,\\
		z_1(0) - x_3\\
		z_2(0) - 1,
	\end{cases} 
\end{equation}
where $z_1$ and $z_2$ are functions of time $t>0$ and $x\in \Lambda$.

We define the physical characteristics to be the two rate parameters and the initial value for $z_1$ valued in the parameter domain, 
\[
x  = (x_1\; x_2\;  x_3)^\top \in \Lambda = [0.7,1.5]\times [2.75,3.25]\times[1,2] \subset \mathbb{R}^3,
\]
chosen following (\cite{estep08}). System response is quantified using the average of the sum of the reactants over a fixed time interval $[0,T]$,
\[
Q(x) = \frac{1}{T}\int_{0}^{T}(z_1(t,x) + z_2(t,x))dt,
\]

We illustrate by fixing $x_3=1.65$ and plotting the nonlinear decision boundary for the Brusselator model corresponding to $q_0=3.75$ and $T=5$ in Fig.~\ref{fig:Brusscontour}. The estimated probability of failure is $0.30026$ computed using the proposed algorithm (using $5000$ points and $20$ repetitions for later reference).
\begin{figure}[htb]
	\centering
	\includegraphics[width=0.45\textwidth]{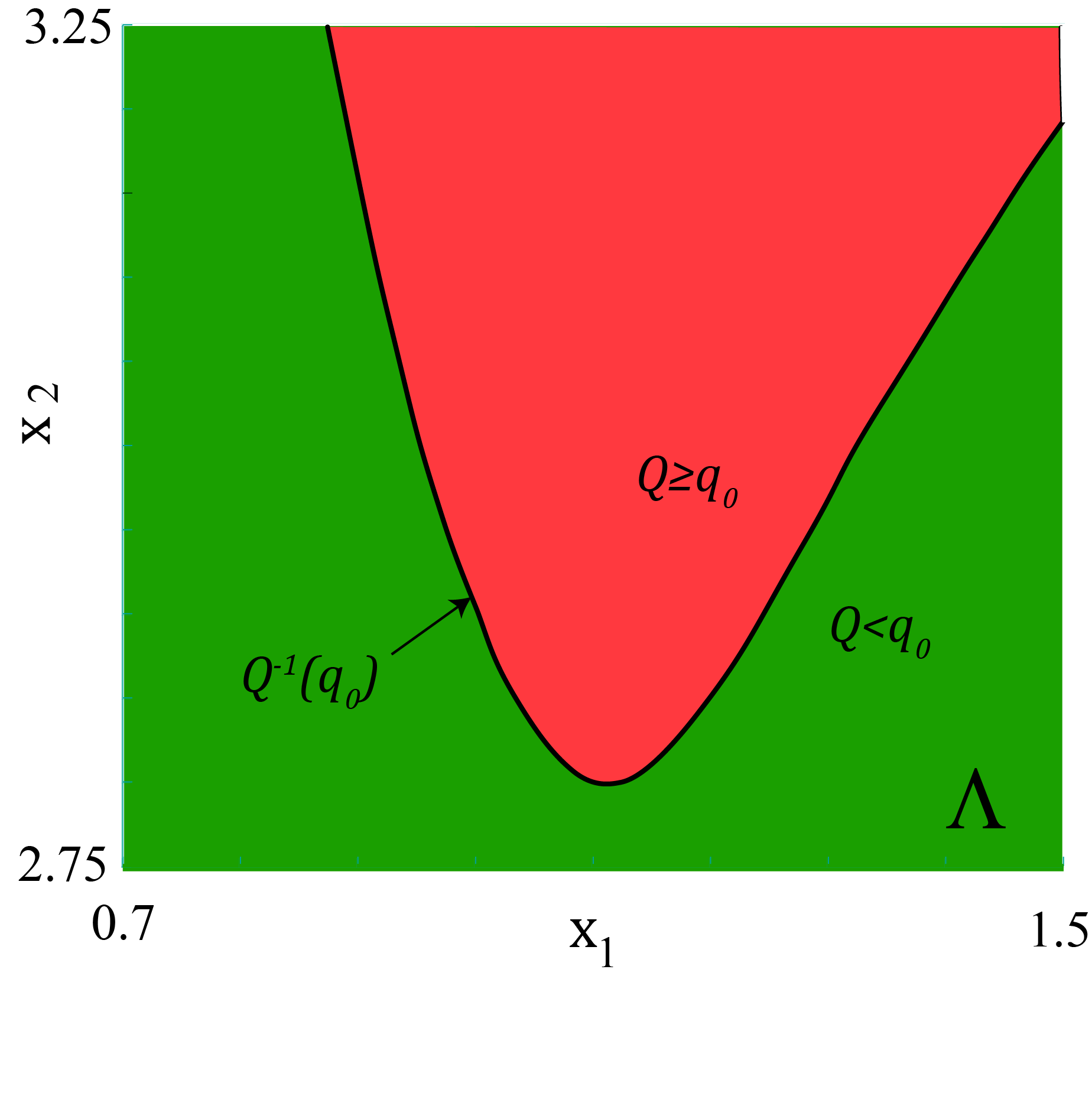}
	\caption{An example of a nonlinear decision boundary for $Q\geq 3.75$ the Brusselator model with fixed  $x_3=1.65$.}
	\label{fig:Brusscontour}
\end{figure}

\subsection*{{Estimating the probability of failure by classification}}

{The direct Monte Carlo approach to compute the probability of failure involves drawing $N$ uniformly distributed random points in $\Lambda$, classifying the points after evaluating $Q$, and using the estimate,
\begin{equation}\label{estProbFail}
P(\Lambda_{\mathcal{F}})\approx \frac{N_{\mathcal{F}}}{N},
\end{equation} 
where $N_{\mathcal{F}}$ is the number of the points that lie in $\Lambda_{\mathcal{F}}$. The Law of Large Numbers and Central Limit Theorem guarantee the estimate converges to the true value in the limit of a large number of samples (\cite{PT_book}). }

{The nonlinear decision boundary  has codimension smaller than the dimension of $\Lambda$, so many points in a sample are expected to lie far from the boundary and   provide limited information regarding the boundary geometry. One consequence is that a large number of samples may be required to achieve accuracy using \eqref{estProbFail}. However, the model $Q$ is very expensive to evaluate in many applications and therefore it is computationally infeasible to evaluate $Q$ at many points. These issues worsen as the dimension of the parameter space increases. It is natural to consider machine learning approaches that classify points in the domain after building a model for classification that requires evaluating $Q$ at relatively few training points.} 

Machine learning methods for classification have evolved from simple linear models to complex ensemble and deep learning techniques. Early approaches like logistic regression (\cite{logistic_regression}) focus on computing linear decision boundaries that offer simplicity and interpretability but struggle with accuracy in the case of nonlinear decision boundaries. Nonlinear methods such as support vector machines with the kernel trick address these limitations by offering the possibility of approximating complex boundaries at the cost of greatly increased complexity of computation. Tree--based methods, including decision trees, and random forests (\cite{random_forest}), have become popular for their ability to handle structured data and nonlinear relationships while providing high interpretability. Probabilistic models, such as Naive Bayes and Bayesian networks (\cite{bayesNet}), predict class probabilities but often require strong assumptions about feature independence. Ensemble methods like bagging, boosting, and stacking, e.g., XGBoost (\cite{xgBoost}), combine multiple classifiers to enhance robustness and accuracy, while deep learning models, including multilayer perceptron (\cite{neuralNet}), often excel in high dimensional, unstructured data such as images and text.  

{The vast majority of the research on machine learning methodology for classification has been driven by data-focused problems in which there is no underlying process model that determines an explicit nonlinear decision boundary. Determining optimal accuracy for a given data set relies heavily on cross-validation  to tune hyperparameters in a given approach, and little attention is paid to fidelity to an underlying true decision boundary. This is reasonable in problems for which it is not possible to access the boundary.  }

{However, the application to computation of probability of failure is fundamentally different in this regard. Theoretically, the process model can be used to determine the actual decision boundary, and that boundary typically has important physical meaning. Thus, interpretability of a surrogate boundary determined by a machine learning methodology becomes an attractive feature.}

{We illustrate in Figure~\ref{fig:interpret}, where we show the surrogate classification boundary along with the classification results computed using the \textit{Profile SVM (PSVM)} method  with MagKmeans clustering approach described in Appendix~\ref{App:PSVM} on one of the test problems described in Appendix~\ref{App:PSVM_comp}. Detailed comparisons of performance are discussed in Section~\ref{sec:PSVMG}.  }

{We color the true decision boundary using a dark red color and shade the region of failure. PSVM produces a collection of locally linear surrogate decisions boundaries and classification of a point is determined by a weighted average of classifications using the surrogate boundaries. The optimal classification is obtained by cross-validation. In this example, PSVM determines six linear surrogate boundaries. We indicate the classification of points using stars and circles. We see that PSVM achieves reasonable accuracy in classification after the use of cross-validation, but the associated surrogate decision boundary has poor interpretability. The point here is to show that achieving a high level of accuracy for a given data set is no guarantee of producing a surrogate decision boundary that is interpretable.}

\begin{figure}[htb]
	\centering
	\includegraphics[width=.4\textwidth]{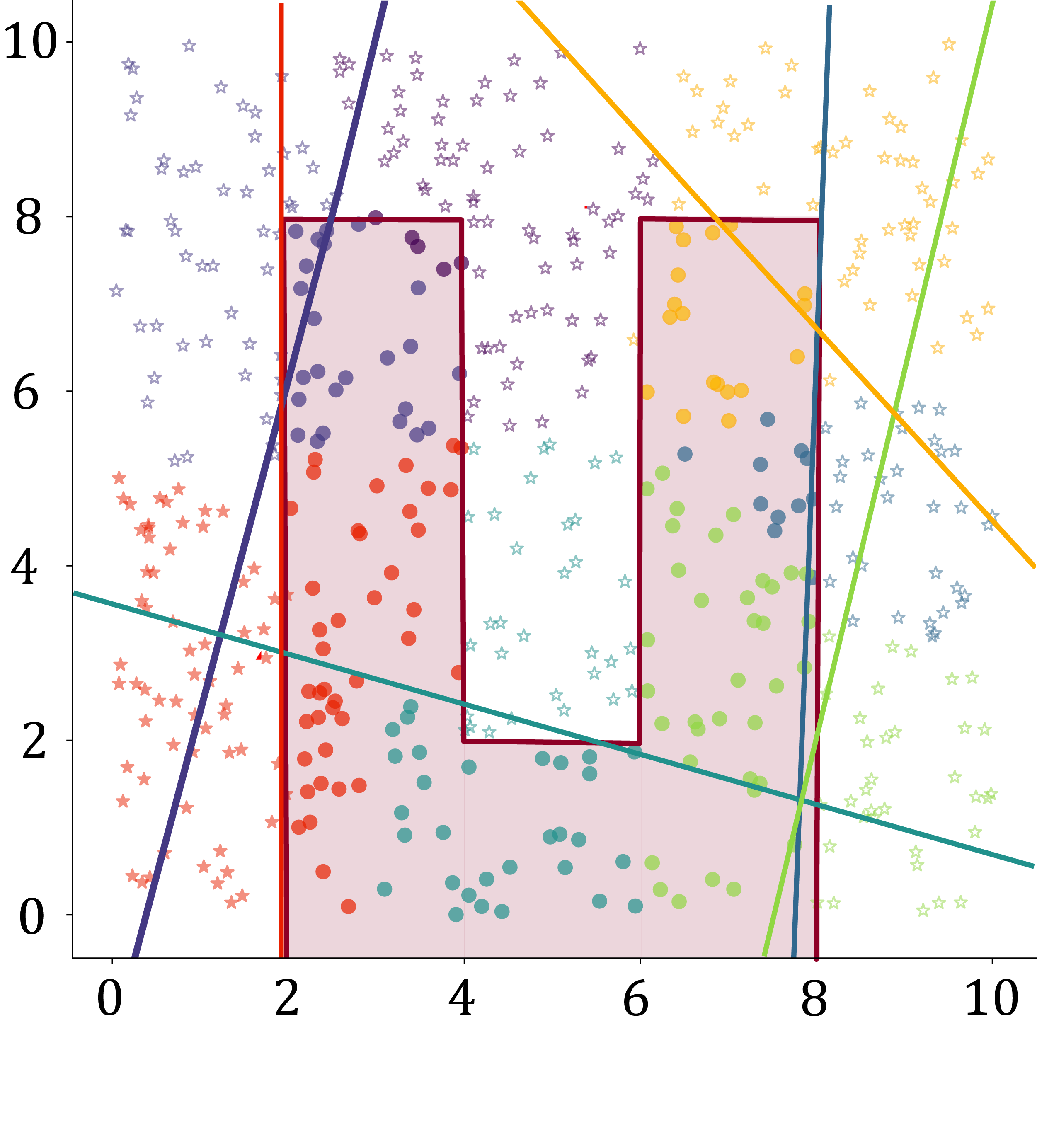}
	\caption{{The surrogate classification boundary along with the classification results computed using the PSVM  with MagKmeans clustering approach described in Section~\ref{App:PSVM} on one of the test problems described in Section~\ref{sec:PSVMG}. The classification of points is indicated using stars and circles. The true decision boundary is colored dark red and the region of failure is shaded.}}
	\label{fig:interpret}
\end{figure}

\subsection*{{A new approach}}
{We address these challenges by introducing a novel machine learning approach called the \emph{Penalized Profile Support Vector Machine based on the Gabriel edited set}. Our methodology is aimed at establishing a surrogate decision boundary that approximates the true nonlinear decision boundary in an effort to increase interpretability, especially on small training sets, while providing competitive accuracy. The methodology has two components: ($\bullet$) It uses an adaptive sampling strategy designed to strategically allocate points near the decision boundary; ($\bullet$) It uses derivative information to construct a locally linear classification boundary that remains consistent with the geometry of the true nonlinear decision boundary computed. The method achieves robust accuracy comparable to  other state of the art classification methods while being distinguished by the simplicity and fidelity of the surrogate decision boundary.}

\subsection*{Outline}

In Section~~\ref{sec:setup}, we describe relevant work on which we build our methodology. In Section~~\ref{sec:method}, we present our methodology, and we present two convergence results in Section~~\ref{sec:proof}. We show results from a numerical simulation study in Section~~\ref{sec:simul}. {We present an application to the Lotka--Volterra model for competing species in Section~~\ref{sec:LVM}.} In Section~~\ref{sec:conclu}, we present a conclusion. We present details, descriptions of other methodologies, and algorithms in the appendices.

\section{{Ingredients for constructing the surrogate decision boundary}}\label{sec:setup}

{We describe two ingredients common to the our classification methods. Additional details are provided in Appendix~\ref{sec:details}.}

\subsection{{Probability of Failure--Darts sampling}}\label{sec:POF}

{We consider approaches that build a classification model using a set of training points. In our approach, we construct the collection of training points using  an adaptive sampling strategy  designed to efficiently concentrate  points near the decision boundary  called \textit{Probability of Failure -- Darts (POF--Darts)} (\cite{POF-darts,Ebeida_Kd_Darts}).  
 POF--Darts determines a collection of sample points $\{x_i \}_{i=1}^n\subset \Lambda$. Each point is associated with a sphere with radius determined by an estimate of the distance to the nonlinear decision boundary. The sampling strategy spaces the points to avoid overlap of the associated spheres. The samples are computed along specified hyper-planes  (``darts'') to increase efficiency.}
 
 {POF--Darts tends to concentrate new samples in regions where the nonlinear decision boundary is poorly described by the existing samplej. Sample sets generated by POF--Darts tend to have a high density of sample points in a small neighborhood of the nonlinear decision boundary, see Section~~\ref{sec:proof}. To illustrate, we fix $x_3=1.65$, $q_0=3.75$, $T=5$ and initial samples = $10$ in the Brusselator model and show a set of POF--Darts samples in Fig.~\ref{fig:POF-bruss}.  We can see the characteristic high density of sample points in a neighborhood of the nonlinear decision boundary.}
 
 \begin{figure}[htb]
 	\centering
 	\subfloat[100 samples.]{\includegraphics[width=0.5\textwidth]{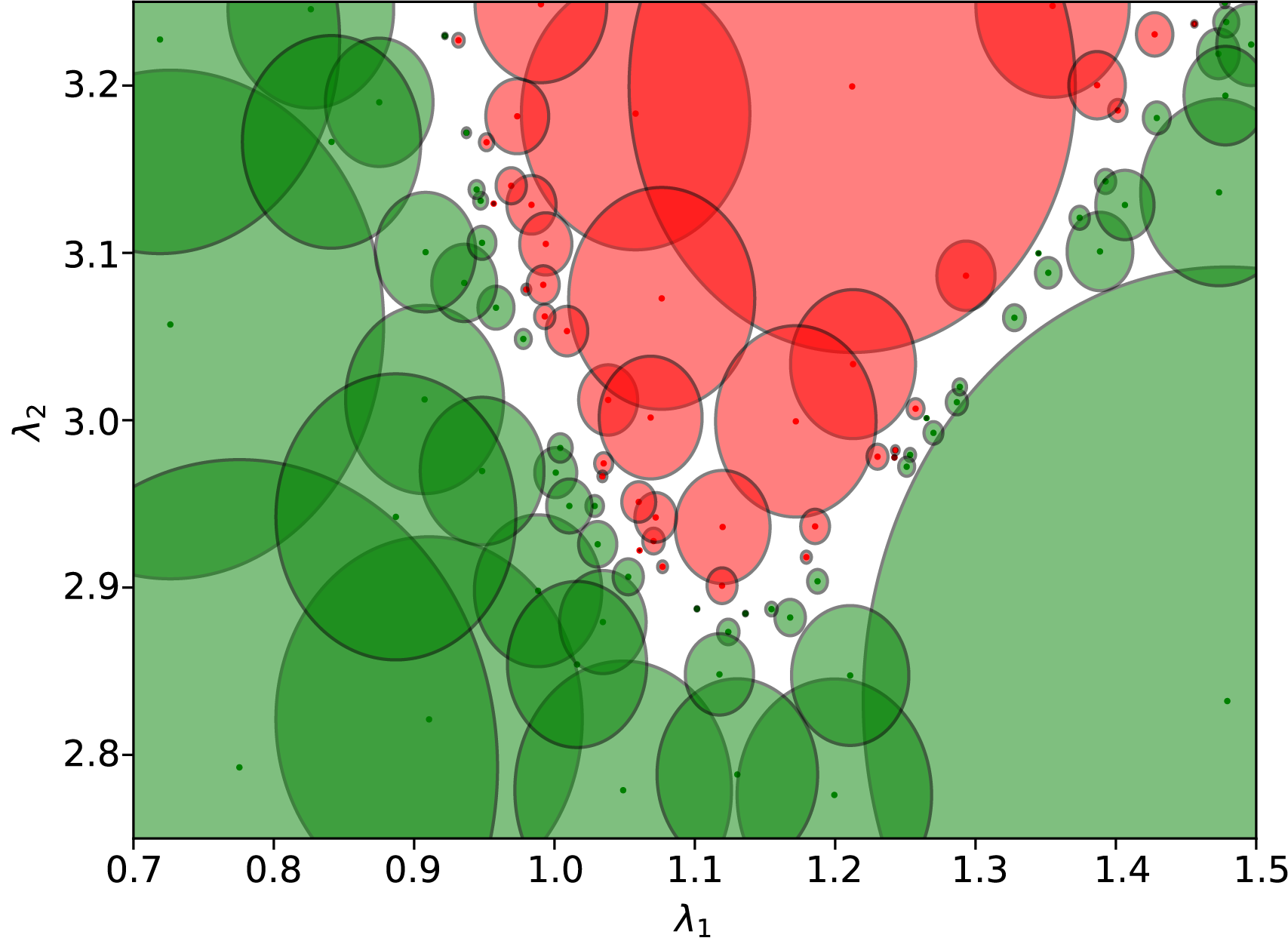}}\hfill
 	\subfloat[500 samples.]{\includegraphics[width=0.5\textwidth]{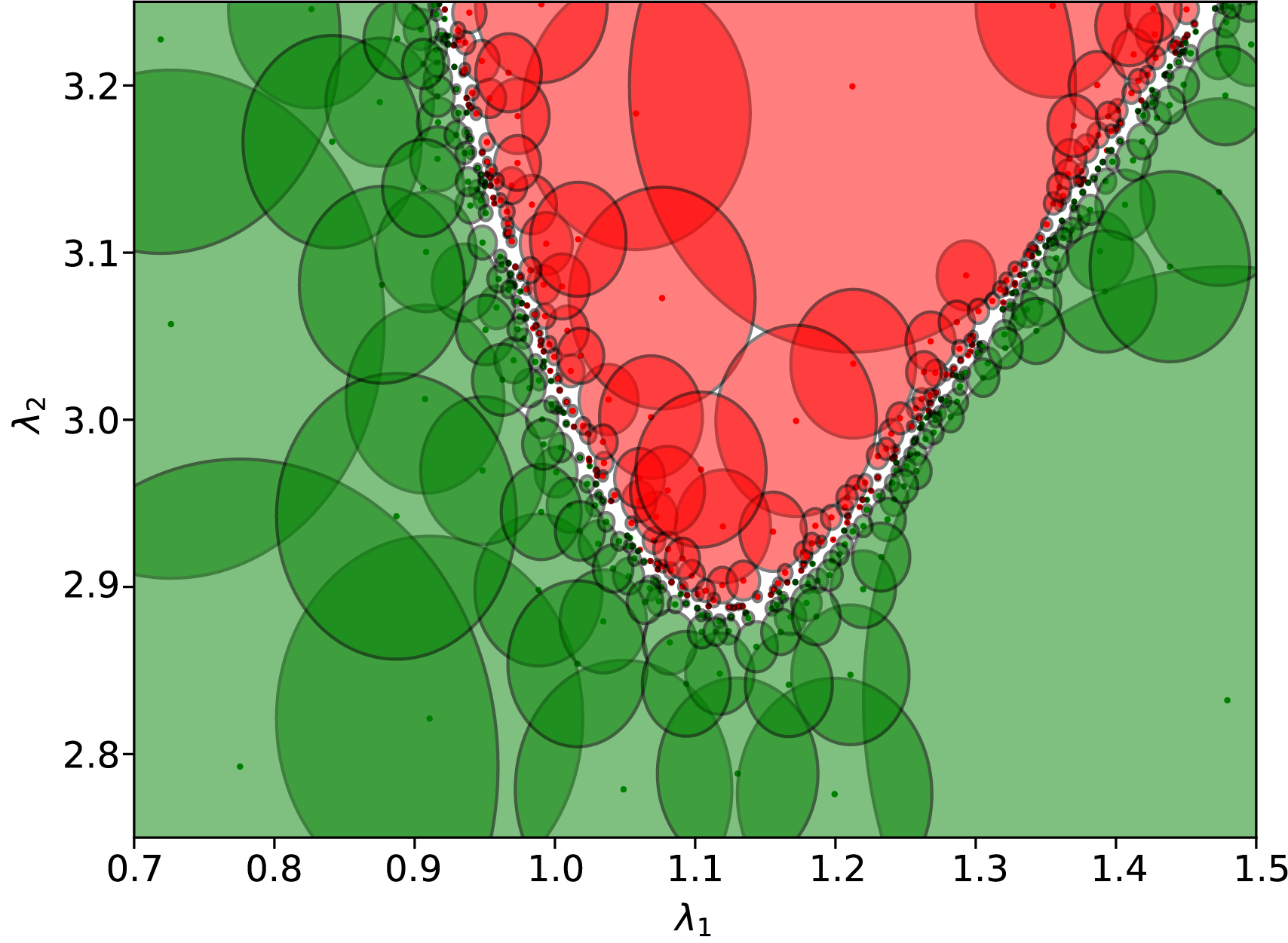}}\\
 	\caption{Examples of sample points and their corresponding spheres produced through POF--Darts sampling for the Brusselator model. Samples in negative and positive regions are colored red and green respectively.}
 	\label{fig:POF-bruss}
 \end{figure}
 
One consequence of concentrating sample points in a neighborhood of the nonlinear decision boundary is that the minimum distance between points in different classification sets decreases as the number of samples increases. {This is expected to improve  performance for many classification methods.}
 
{We use POF--Darts with one dimensional darts to generate the training set. In addition to the dimension of the darts, POF--Darts depends on an initial set of sample points and a parameter $\mathcal{L}$ that is related to an estimate of the distance between points in the initial set to the nonlinear decision boundary.} Details for POF--Darts are provided in Appendix~\ref{det:POFDarts} and we present Algorithms \ref{alg:POF} and \ref{alg:1ddart} in Appendix~\ref{Algs}.

\subsection{{The Gabriel edited set and Characteristic Boundary Points}}

{Our classification methodology uses the points from a training set that lie closest to the nonlinear decision boundary. Intuitively, the boundary is located within the region circumscribed by these points.  As the number of points in a sequence of  training sets increase, the circumscribed regions become ``narrower'' while the gaps between points on the boundary become smaller, thus suggesting that the nonlinear decision boundary is increasingly well described, see Section~~\ref{sec:proof}. }

{However, the points closest to the boundary must be determined without knowing the location of the boundary.  Our approach for this problem uses the Gabriel edited set. The \textit{Gabriel graph} has the property that two points are connected by a line segment in the graph if and only if the sphere containing the segment connecting the points as a diameter does not intersect any other points in the sample and the \textit{Gabriel neighbors} in a Gabriel graph are the points connected by a line segment in the graph and the \textit{Gabriel edited set (GES)} comprises Gabriel neighbors with different classifications. }

To be specific, we denote the \textit{training set} by $\{\mathcal{X}, \mathcal{Y}\}$ =$\{\mathcal{X}, \mathcal{Y}\}$ = $\{x_i,y_i\}_{i=1}^n$, where $\{x_i\}_{i=1}^n \subset \Lambda$ are samples in $\Lambda$ and we indicate classification of $x_i$ by setting $y_i = \pm 1$ for all $i$. We  denote the Gabriel neighbor pairs in the GES that have different labels by $\mathcal{G} =\{(x_i,x_j)\}_{(i,j)\in \mathcal{N}}$, where 
\[
\mathcal{N}=\{(i,j)\, |\, x_i \text{ and } x_j \text{ are Gabriel neighbors, } 1 \leq i \leq n, \, i < j \leq n, \, y_i \ne y_j\}
\] 
denotes the set of paired indices of Gabriel neighbors in $\mathcal{X}$.  We denote the individual points in $\mathcal{G}$ by $\{\mathcal{X}_{\mathcal{G}}, \mathcal{Y}_{\mathcal{G}}\} = \{x_i,y_i\}_{i \in \widehat{\mathcal{N}}}$, where $\widehat{\mathcal{N}} = \{i\,|\, (i,j)\in\mathcal{N} \}\bigcup \{j\,|\, (i,j)\in\mathcal{N} \}$. 

In addition to the GES, we use the \textit{Characteristic Boundary Points (CBPs)} $\overline{\mathcal{G}}=\{\overline{x}_{ij}\}_{(i,j) \in {\mathcal{N}}}$, where $\overline{x}_{ij} =\frac{1}{2}\big(x_i + x_j \big)$ for $(i,j)\in\mathcal{N}$, which are the midpoints between Gabriel neighbors within the GES that have different labels. 

We present the algorithm for determining the GES in Algorithm~\ref{alg:improvedGEA} in Appendix~\ref{Algs}. The GES is related to the Voronoi edited set, see Appendix~\ref{det:GES}.  The GES and CBP are a central feature of a classification method introduced by Pujol, et al (\cite{pujol2009}) that we call ``Pujol's method'', see Appendix~\ref{det:Pujol}. We use Pujol's method as a primary benchmark.

\section{{Profile and Penalized Profile Support Vector Machines based on the Gabriel edited set}}\label{sec:method}

In this section, we develop the two proposed classification methodologies. Using a single linear surrogate decision boundary for a nonlinear decision boundary tends to introduce significant misclassification. For this reason, we consider locally (piecewise) linear surrogate boundaries. The intuition is that a nonlinear decision boundary is approximately linear in sufficiently localized regions. Indeed, if the nonlinear decision boundary is known, then the approximation computed from tangent planes at a sufficiently dense set of points on the boundary provides the most efficient asymptotically accurate locally linear surrogate decision boundary.  The primary difference between the two methods is how the surrogate boundaries are constructed.

\subsection{The Profile SVM using the Gabriel edited set model}\label{sec:PSVMG}

{The first method is called the \textit{Profile SVM using the Gabriel edited set (PSVMG)}.  We build a local linear surrogate decision boundary using clusters of GES points constructed by clustering the CBPs using Lloyd's k-means algorithm (\cite{lloyd1982least}).  As each CBP is associated with a pair of Gabriel neighbors,  the constructed clusters of GES points always contain points from different classes. To avoid solving a complex optimization problem, we  fix the number of clusters  before the construction begins. }

%To be precise, we denote the points in $\overline{\mathcal{G}}$ by $\{\overline{x}_i\}_{i=1}^m$. We let $Z$ denote an $m \times k$ matrix where $k$ is the number of clusters. The $i^{th}$ row of $Z$ indicates the membership of $\overline{x}_i$, $1\leq i\leq m$. If $\overline{x}_i$ is part of the $j^{th}$ cluster then $Z_{i,j}=1$ otherwise  $Z_{i,j}=0$. Lloyd's k-means clustering algorithm starts by randomly picking $k$ points in $\overline{\mathcal{G}}$  as cluster centroids $\mathcal{C} = \{C_i\}_{i=1}^k$, then assigns the points to each cluster by minimizing the intra-cluster distance,
%\begin{equation*}
%	\argmin_{Z} \sum_{i=1}^{n}\sum_{j=1}^{K} Z_{i,j}\,\|x_i - C_j\|^2.
%\end{equation*}
%Following this, the centroids are recomputed and new clusters are computed. This iterated  until convergence to a steady state (\cite{lloyd1982least}).

{A well known issue with Lloyd's k-means clustering algorithm is that it may yield clusters that are imbalanced in terms of classification, which affects the accuracy of a locally linear surrogate decision model. We describe the alternative MagKmeans clustering algorithm used for PSVM in Appendix~\ref{App:PSVM}.} 

To build the locally linear surrogate decision boundary, we use ``soft margin'', or ``penalized'', linear  \textit{Support Vector Machines (SVMs)} which allow for the misclassification of trained points while seeking to find an optimal surrogate decision boundary. We define the hyperplanes using vectors ${w} $ and $b$,
\begin{gather*}
	L_1=\{x:  {w}^\top {x} + b = 0\},\\
	L_2=\{x: {w}^\top {x} - b = 0\},
\end{gather*}
with the margin of the region between $L_1$ and $L_2$  defined $|2b| / \|{w}\|$, where $\|\quad\|$ is the Euclidean norm. The hyperplanes are determined by the optimization problem,
\begin{equation}\label{softSVM}
	\min \frac{1}{2}  {w}^\top \, {w} + \beta\sum_{j = 1}^{n} \xi_j  \text{ subject to } y_i {w}^\top\, {x_i} \ge 1 - \xi_i \text{ and } \xi_i \ge 0,\; \forall i,
\end{equation}
where $\beta$ is a regularization parameter and $\xi_j$ represents a distance of $x_j$ to the wrong side of the margin for $1 \leq j \leq n $ ($\xi_j = 0$ if the point is on the correct side). The solution is found using Lagrange multipliers and Calculus as usual, see \cite{ZhuXuankang2022Elaf}. The training points closest to the nonlinear decision boundary are called the \textit{support vectors}. We present the the algorithm for computing the PSVMG model in Algorithm~\ref{alg:G-PSVM} in Appendix~\ref{Algs}. More details can be found in (\cite{ZhuXuankang2022Elaf}).

{Given a collection of SVM, the classification of a point in $\Lambda$ is computed using an ensemble average in which the weights are based on evaluating the Euclidean distances between the point being classified and the centroids of the clusters, with lower weights associated with larger distances. Roughly speaking, the formula is
\[
\textrm{classification of a point}  = \frac{\sum_{\textrm{clusters}} \frac{\textrm{classification by cluster SVM}}{\textrm{distance to cluster}}}{\sum_{\textrm{clusters}}\textrm{(distance to cluster)}^{-1}}.
\] }

\begin{figure}[htb]
	\centering
	\subfloat[Training set and the true boundary.]{\includegraphics[width= 6cm]{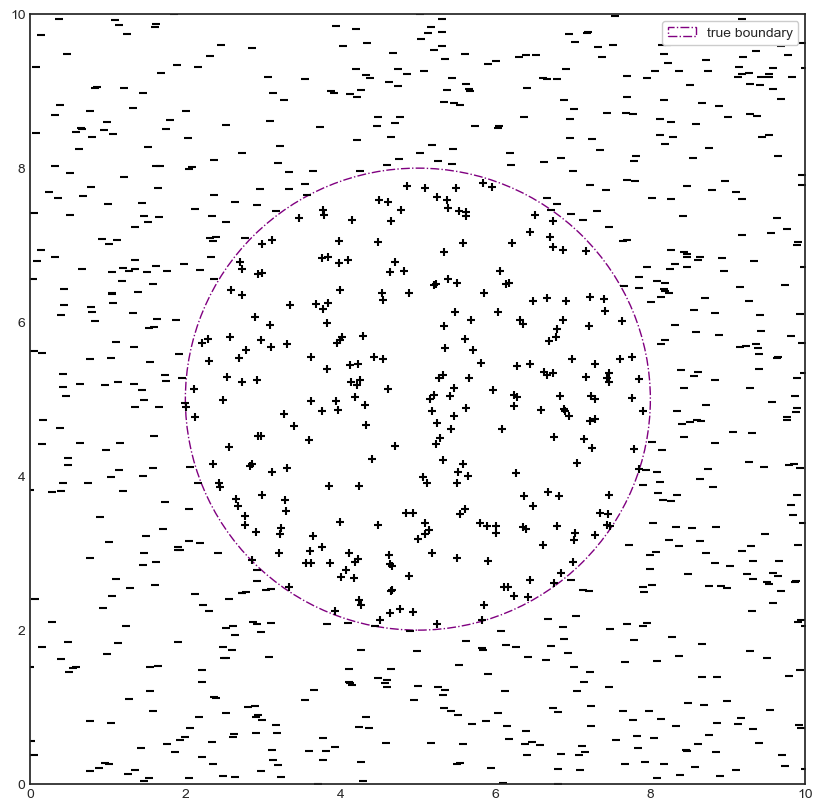}}
	\subfloat[The GES and the CBPs.]{\includegraphics[width= 6cm]{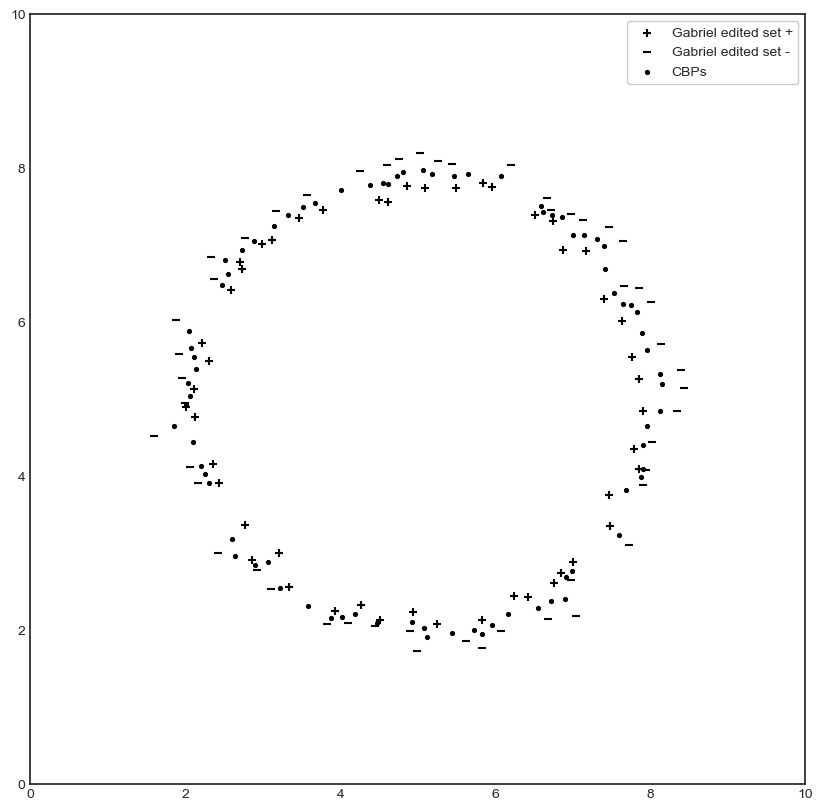}}\\
	\subfloat[8 clusters generated using the CBPs.]{\includegraphics[width= 6cm]{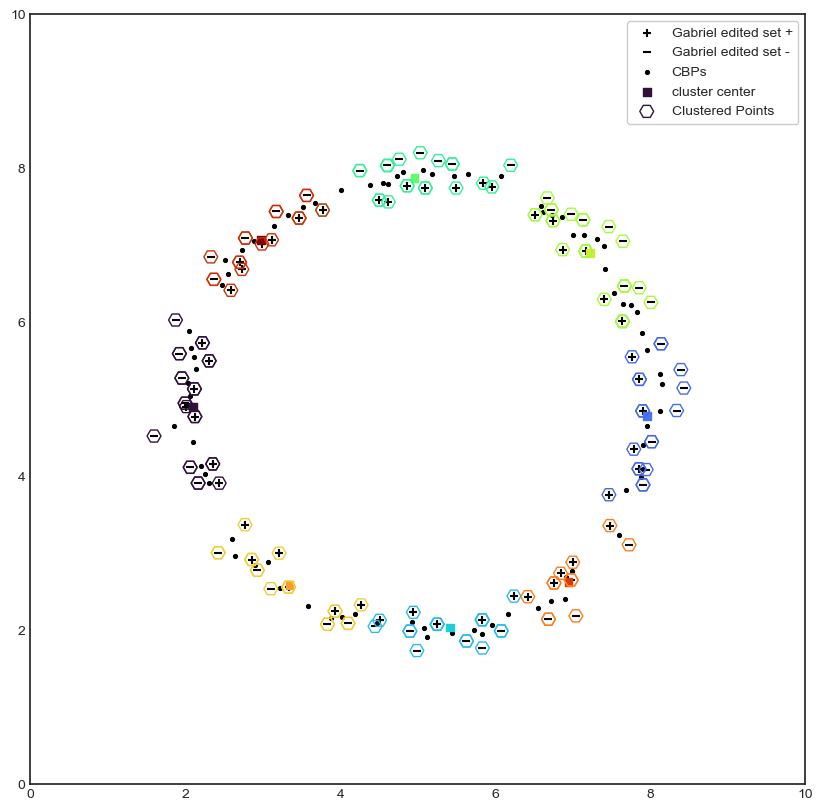}}
	\subfloat[Linear SVMs built on the clusters.]{\includegraphics[width= 6cm]{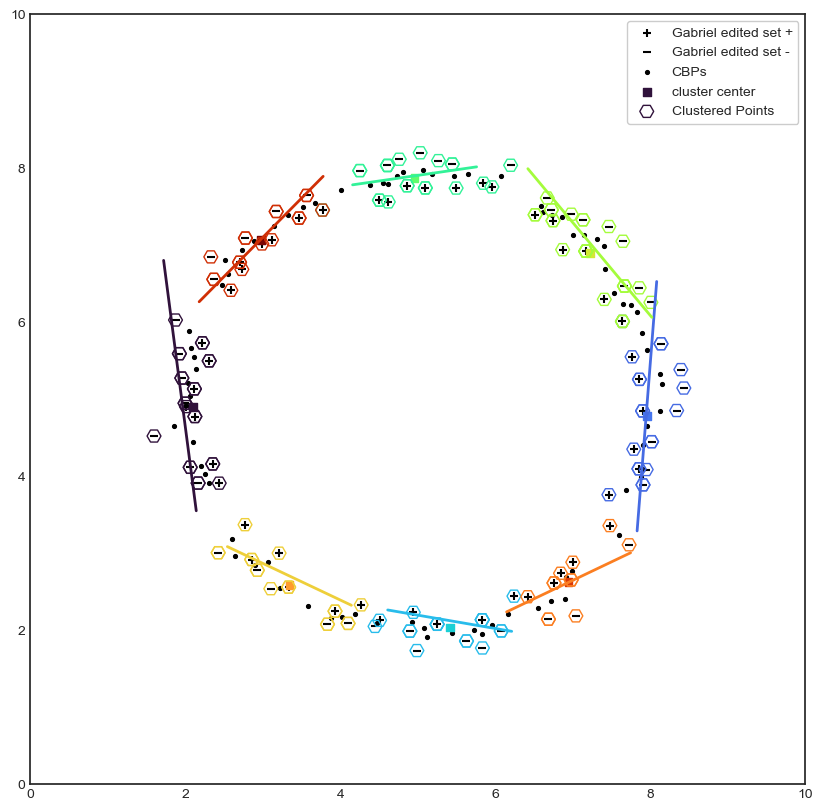}}\\
	\caption{Example of the performance of PSVMG on a circular decision boundary.}
	\label{fig:diagram GPSVMa}
\end{figure}

{PSVMG  produces clusters that are generally aligned with the nonlinear decision boundary. We use a circular decision boundary to illustrate the performance of PSVMG. We select $500$ samples uniformly from $\{x_i,y_i: x_i \in [0,10]\times[0,10]$ with $y_i = -1 $ if $ (x_{i1} - 5)^2 + (x_{i2} - 5)^2 \ge 9$ and $1$ otherwise$\}$.  In Fig.~\ref{fig:diagram GPSVMa}~(a), we show the circular boundary and the classification of the points. In Fig.~\ref{fig:diagram GPSVMa}~(b), we show the GES and CBPs. The eight clusters are shown in Fig.~\ref{fig:diagram GPSVMa}~(c) and the corresponding linear SVM models are shown in Fig.~\ref{fig:diagram GPSVMa}~(d). The clusters are roughly the same size. The size is  too large for a linear model to have high fidelity with the curved decision boundary. In  Fig.~\ref{fig:diagram GPSVMb}, we show the approximate decision boundary and the corresponding partition of $\Lambda$.}

\begin{figure}[htb]
	\centering
\includegraphics[width= 6cm]{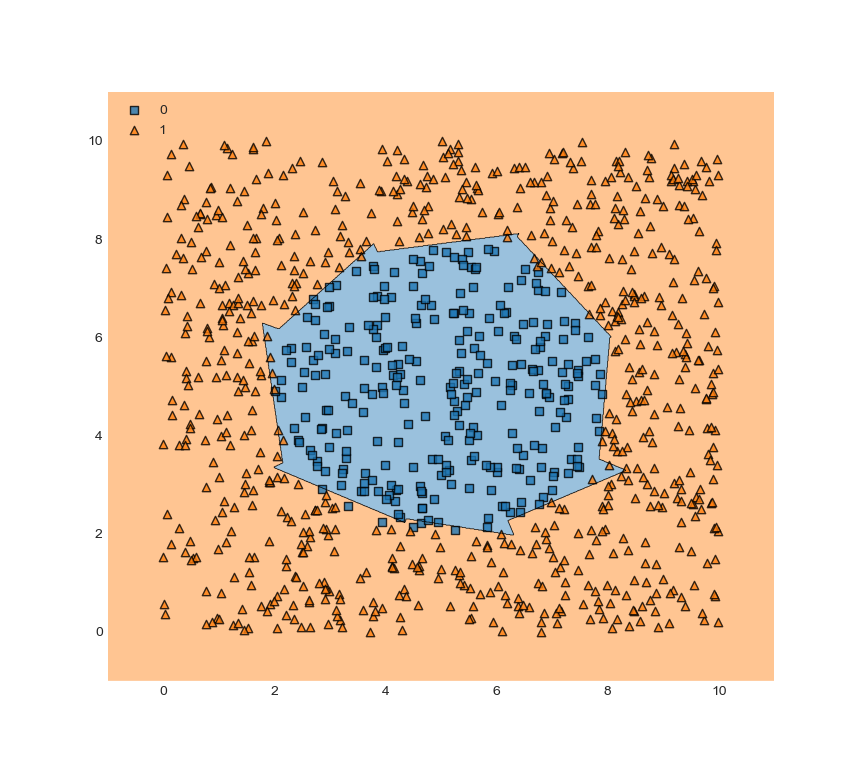}
	\caption{Decomposition of $\Lambda$ by classification from the PSVMG model for a circular decision boundary with labels colored in blue and orange.}
	\label{fig:diagram GPSVMb}
\end{figure}

We further illustrate in Fig.~\ref{fig:GPSVM_comparison}, where we show the PSVMG surrogate boundaries for three synthetic examples defined in Appendix~\ref{App:PSVM_comp}. The second example is complicated because the curvature of the nonlinear boundary varies significantly while the last two examples are difficult because of complex pattern of transitions between the regions of different classifications. For each example, we choose a training set of $500$ points distributed uniformly in $\Lambda$ and we set $\gamma = 0.5, \beta = 0.1$. 
For comparison,  we show the corresponding surrogate boundaries determined by PSVM with MagKmeans clustering (Appendix~\ref{App:PSVM}) in Fig.~\ref{fig:PSVM_comparison}.  We initiate PSVM  using a varying number of initial clusters. We see that the surrogate boundaries determined by PSVMG follow the geometry of the nonlinear decision boundary with relatively high fidelity compared to PSVM.

\begin{figure}[htb]
	\centering
	\includegraphics[width=.9\textwidth]{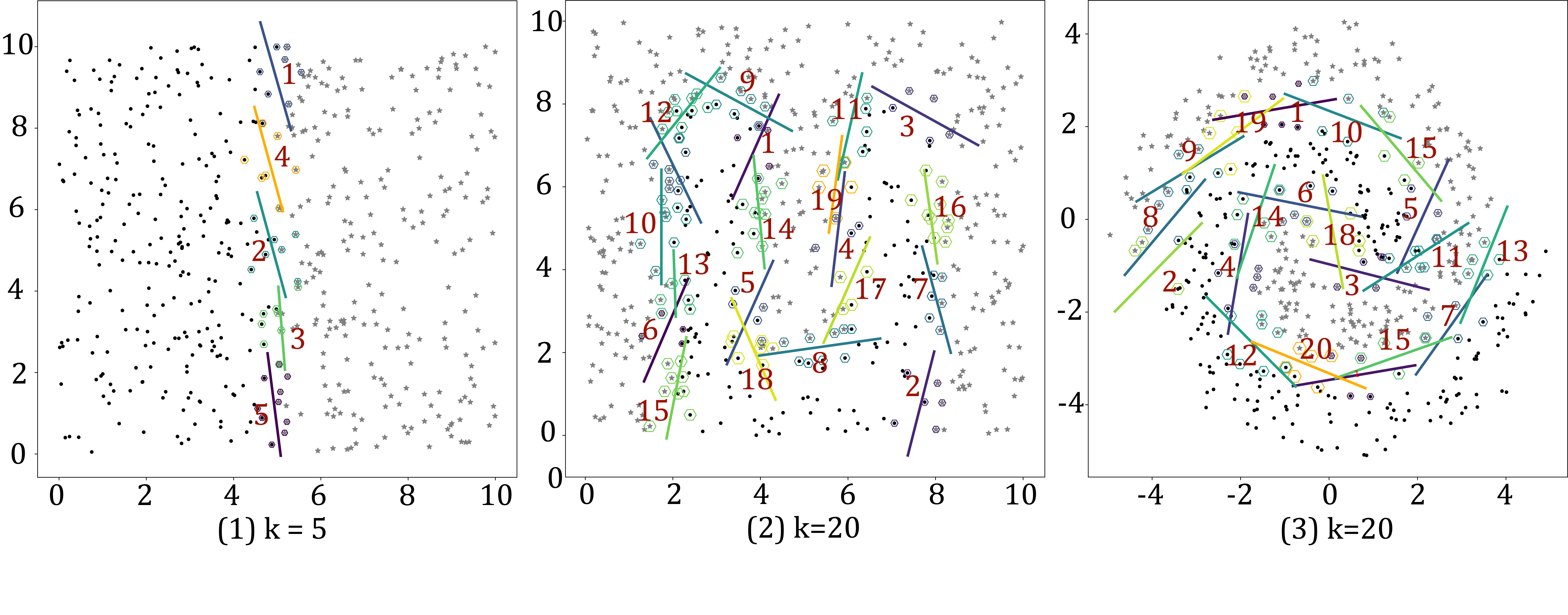}
	\caption{The approximate local linear SVM boundaries generated from PSVMG  with the indicated initial number of clusters $k$ for the synthetic examples defined in Appendix~\ref{App:PSVM_comp}. Compare to Fig.~\ref{fig:PSVM_comparison}.}
	\label{fig:GPSVM_comparison}
\end{figure}

\begin{figure}[htb]
	\centering
	\includegraphics[width=\textwidth]{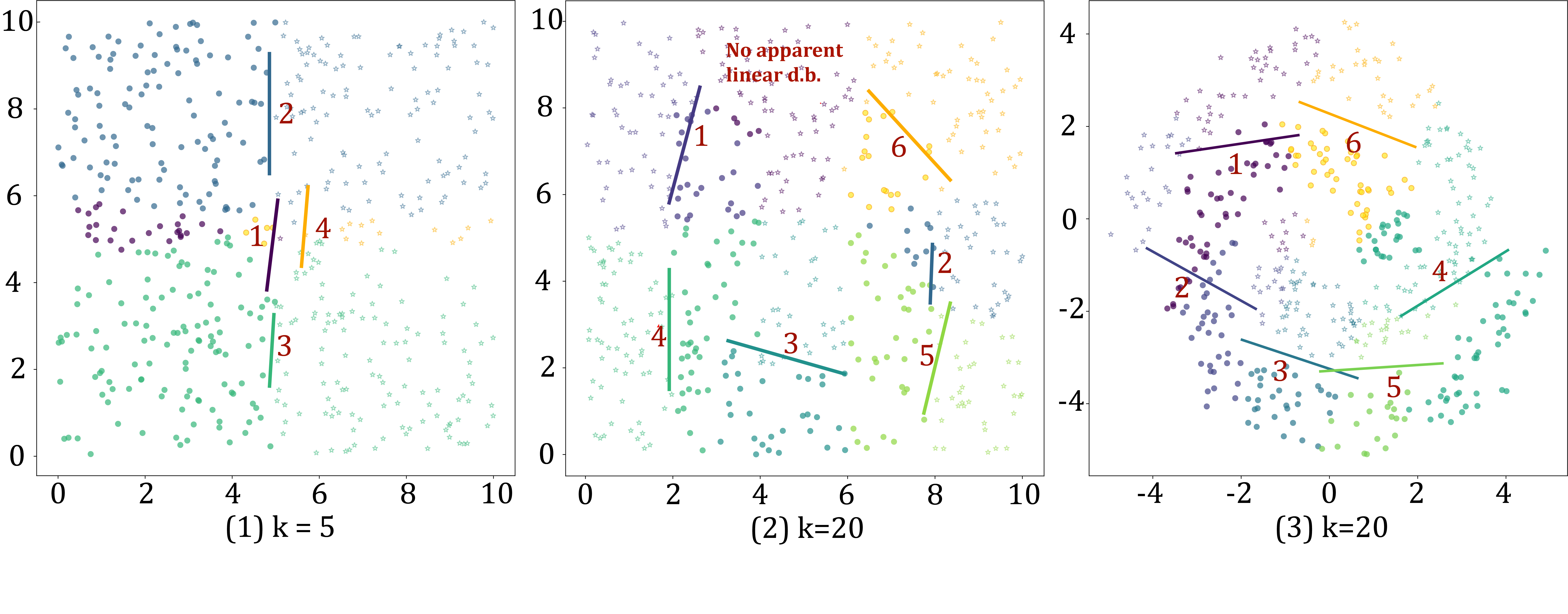}
	\caption{Typical approximate local linear SVM boundaries generated from PSVM with MagKmeans clustering starting with the indicated initial number of clusters $k$ for the synthetic examples defined in Appendix~\ref{App:PSVM_comp}. Compare to Fig.~\ref{fig:GPSVM_comparison}.}
	\label{fig:PSVM_comparison}
\end{figure}

\subsection{{The penalized Profile SVM using the Gabriel edited set model}}

{PSVMG tends to produce well-separated clusters, leading to poor classification of samples in regions ``between'' clusters. Moreover, the angles between most of the PSVMG SVMs and the local tangent to the nonlinear decision boundary varies more than is desirable. We introduce the \textit{penalized Profile SVM using the Gabriel edited set (PPSVMG)} model primarily  to capture the geometric complexity of the boundary more accurately, provide robustly smooth behavior with respect to clustering, and  maintain computational efficiency.}  

{We begin with an alternative approach to clustering CPBs. Setting each CBP as a cluster centroid,  we use the nearest neighbor rule to find its $K$ nearest CBPs, after which we build a cluster of GES referenced by the CBPs. This step creates more clusters than the previous step, namely $m$ clusters of size $K$.  We introduce an additional step intended to combine clusters that are similar, where two clusters are considered similar if they share a specified proportion (the \textit{similarity parameter}) of training points in common.  Thus, the number of clusters may decrease  in a controlled fashion. We present the algorithm clustering  in Algorithm~\ref{alg:PPSVMG_clustering} in Appendix~\ref{Algs} and illustrate the cluster reduction in Figures~\ref{fig:diagramGMSVM100}--\ref{fig:diagramGMSVM500}.}

{As mentioned, the PSVMG model may yield a locally linear SVM model that is not well aligned with the nonlinear decision boundary. This is an issue in particular when there are relatively few Gabriel neighbors and the line joining two Gabriel neighbors is far from orthogonal to the tangent plane of the nonlinear decision boundary at a nearby point.}

We construct a SVM that includes a penalty term defined to be  $1$ - the squared cosine of the deviation between the linear SVM and the boundary gradient, i.e., 
\[
\text{penalty} =  1 - \frac{ w^{\top} \nabla \overline{Q} \nabla \overline{Q}^{\top} w }{ w^{\top}w },
\]
where $w$ is an estimate of the distance to the nonlinear decision boundary and $\overline{Q}$ is the average mean gradient of the Gabriel edited points belonging to the cluster normalized to be a unit vector. The objective function for the linear SVM becomes,
\begin{multline*}
    \min_{w, b, \{\xi_{j}\} }  \frac{1}{2} w^{\top}w + \beta \sum_{j\in \widehat{\mathcal{N}}} \xi_{j} + \kappa \bigg(1 -  \frac{w^{\top} \nabla \overline{Q}\nabla \overline{Q}^{\top}w}{w^{\top}w} 
  \bigg)\\
     s.t \quad (w^{\top}x_j + b)y_j \ge 1 - \xi_{j},\; \xi_j\geq 0 \quad \forall j \in \widehat{\mathcal{N}},
\end{multline*}
where $\kappa$ is a user defined parameter in $[0,\infty)$. Solving this optimization problem is challenging because it requires finding the optimal solution at a saddle point. We relax the problem by finding the optimal rotation of $w$ with respect to the original SVM to make it more parallel to the tangent plane of the nonlinear decision boundary. We let $\theta$ denote the angle of rotation of $w$ normalized to $[0,1]$. The decision norm $w(\theta)$ after rotation is,
\[
w(\theta) = \theta w^* + (1 - \theta )w_{Q} = \theta (w^* - w_{Q}) + w_{Q}, \quad \theta \in [0,1]
\]
with $ w_{Q} = \frac{w^{*\top}\overline{Q}}{\overline{Q}^{\top}\overline{Q}}\, \overline{Q}$. When $ \theta  = 0 $, $ w^* $ is rotated in the same direction as $ \nabla Q $, and when $\theta  = 1 $, $ w^* $ remains unchanged. The relaxed penalty term is:
\[
\text{relaxed penalty} = 1 - \bigg(\theta \frac{w^{*\top}\overline{Q}^{\top}\overline{Q}w^*}{w^{*\top}w^*} + (1- \theta )\bigg) = \theta  - \theta \frac{w^{*\top}\overline{Q}^{\top}\overline{Q}w^*}{w^{*\top}w^*},
\]
and the primal problem is to minimize:
\begin{multline*}
    \min_{\theta , w, b, \{\varepsilon_{j}\} }  \frac{1}{2} w(\theta )^{\top}w(\theta ) + \beta\sum_{j\in \widehat{\mathcal{N}}} \xi_{j} + \kappa t(1 - \frac{w^{*\top}\overline{Q}^{\top}\overline{Q}w^*}{w^{*\top}w^*})\\
     s.t. \quad (w^{\top}x_j + b)y_j \ge 1 - \xi_{j}, \; \xi(j)\geq 0 \quad \forall j\in\widehat{\mathcal{N}}.
\end{multline*}
The Lagrangian dual form is, 
\begin{align*}
     &\max_{\{\alpha_j\}, \{\mu_j\} } \min_{\theta , w, b, \{\xi_{j}\} } \frac{1}{2} w(\theta )^{\top} w(\theta ) + \beta \sum_{j\in \widehat{\mathcal{N}}} \xi_j + \kappa \theta  \bigg(1 - \frac{w^{*\top}\overline{Q}^{\top}\overline{Q}w^*}{w^{*\top}w^*}\bigg) \\
    & \qquad \qquad \qquad \qquad - \sum_{j\in \widehat{\mathcal{N}}}  \alpha_j \big((w(\theta )^{\top}x_j + b)y_j + \xi_j - 1\big) - \sum_{j\in \widehat{\mathcal{N}}} \mu_j \xi_j\\
    &\qquad \qquad  \qquad \qquad \qquad \qquad \qquad \qquad \qquad \qquad s.t. \quad \quad \alpha_j, \mu_j, \xi_j \ge 0 \quad \forall j \in \widehat{\mathcal{N}}. 
\end{align*}
Setting the derivatives with respect to $\theta $, $\xi_{j}$, and $b$ to $0$ yields,
\begin{equation*}
        w(\theta )^{\top}(w^* - w_{Q}) + \kappa \bigg(1 - \frac{w^{*{\top}}\overline{Q}^{\top}\overline{Q}w^*}{w^{*{\top}}w^*}\bigg) - \sum_{j\in \widehat{\mathcal{N}}} \alpha_j (w^* - w_{Q})^{\top}x_jy_j = 0,
    \end{equation*}
    \begin{equation*}
        (\theta (w^* - w_{Q}) + w_{Q})^{\top}(w^* - w_{Q}) =  \sum_{j\in \widehat{\mathcal{N}}} \alpha_j (w^* - w_{Q})^{\top}x_jy_j - \kappa \bigg(1 - \frac{w^{*{\top}}\overline{Q}^{\top}\overline{Q}w^*}{w^{*{\top}}w^*}\bigg),
    \end{equation*}
    \begin{equation}\label{eqn_t}
        \Rightarrow \theta = \frac{\sum_{j\in \widehat{\mathcal{N}}} \alpha_j (w^* - w_{Q})^{\top}x_jy_j - w_{Q}^{\top}(w^* - w_{Q}) - \kappa (1 - \frac{w^{*{\top}}\overline{Q}^{\top}\overline{Q}w^*}{w^{*{\top}}w^*})}{(w^* - w_{Q})^{\top}(w^* - w_{Q})} .
    \end{equation}
The optimal solution can be found by solving the linear problem:
\begin{align}
    \label{optimization}
    \max_{\alpha} \left(\frac{1}{2} w(\theta )^{\top} w(\theta )  + \sum_{j\in \widehat{\mathcal{N}}} \alpha_j + \kappa \theta   \bigg(1 - \frac{w^{*{\top}}\overline{Q}^{\top}\overline{Q}w^*}{w^{*{\top}}w^*}\bigg)- \sum_{j\in \widehat{\mathcal{N}}} \alpha_j w(\theta )^{\top}x_jy_j\right).
\end{align}
After obtaining the optimal solution for 
$\alpha$, $\theta$ and $ w(\theta)$ can be directly calculated using (\ref{eqn_t}), and the support vectors can be identified by selecting the training points with nonzero $\alpha$. Then, $b$ can be obtained by averaging the values of $b_j$, 
\[
b_j = y_j - w(\theta)^{\top}x_j.
\]

\begin{figure}[htb]
	\centering
	\includegraphics[width=0.8\textwidth]{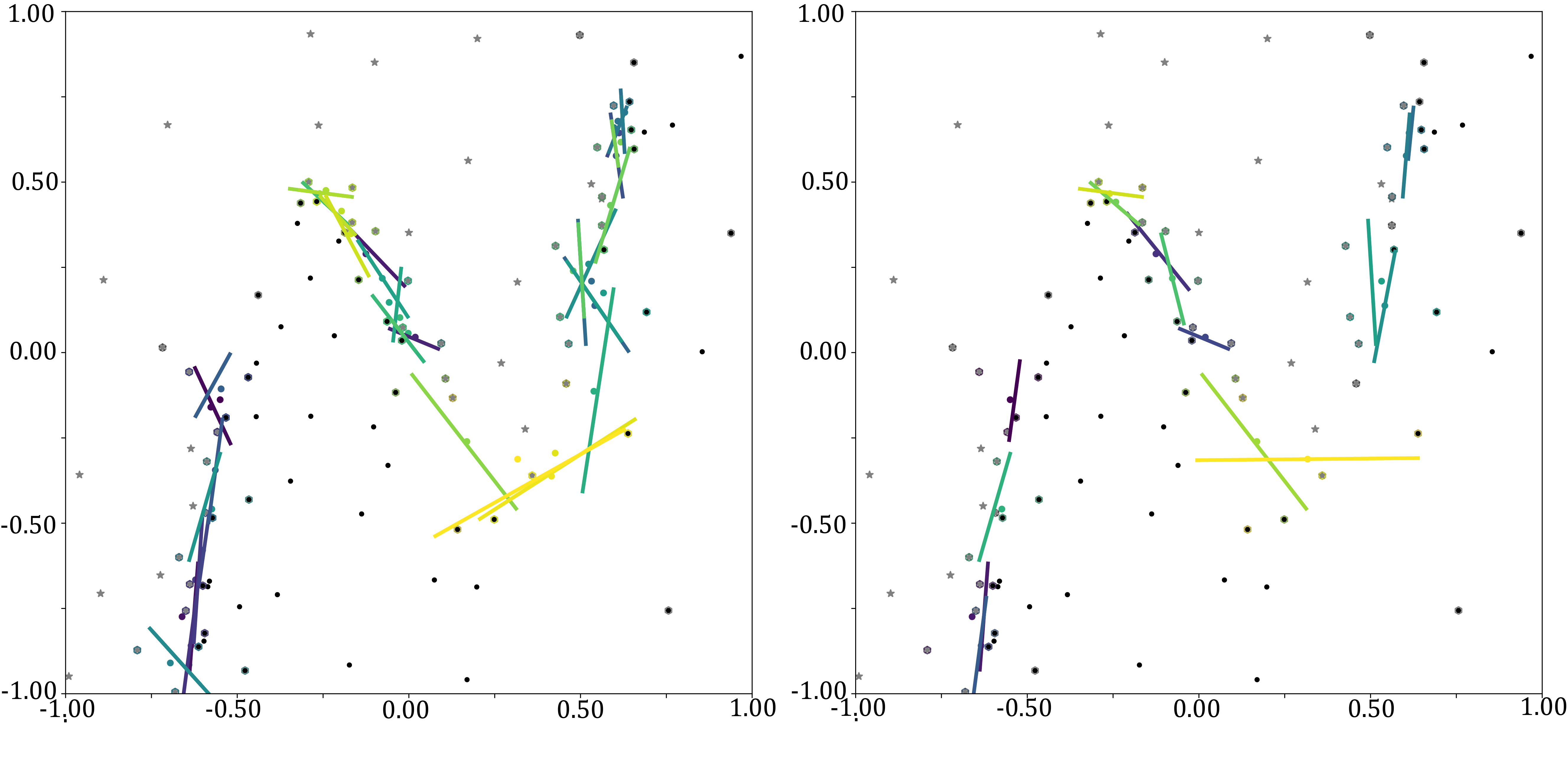}
	\caption{Classification results for the  PPSVMG model with initial cluster size $K=3$ and training set with $n=100$ points. Left: $\kappa= 0$ and similarity = $1$. Right: $\kappa= 10$ and similarity = $.4$. }
	\label{fig:diagramGMSVM100}
\end{figure}

{The penalization stabilizes the behavior of the approximate decision boundary, particularly in situations when the size of data clusters is small and decreases the sensitivity of the linear SVM models to  variations in cluster size and location of cluster point, thereby improving the robustness of classification.}

\begin{figure}[htb]
	\centering
	\includegraphics[width=0.8\textwidth]{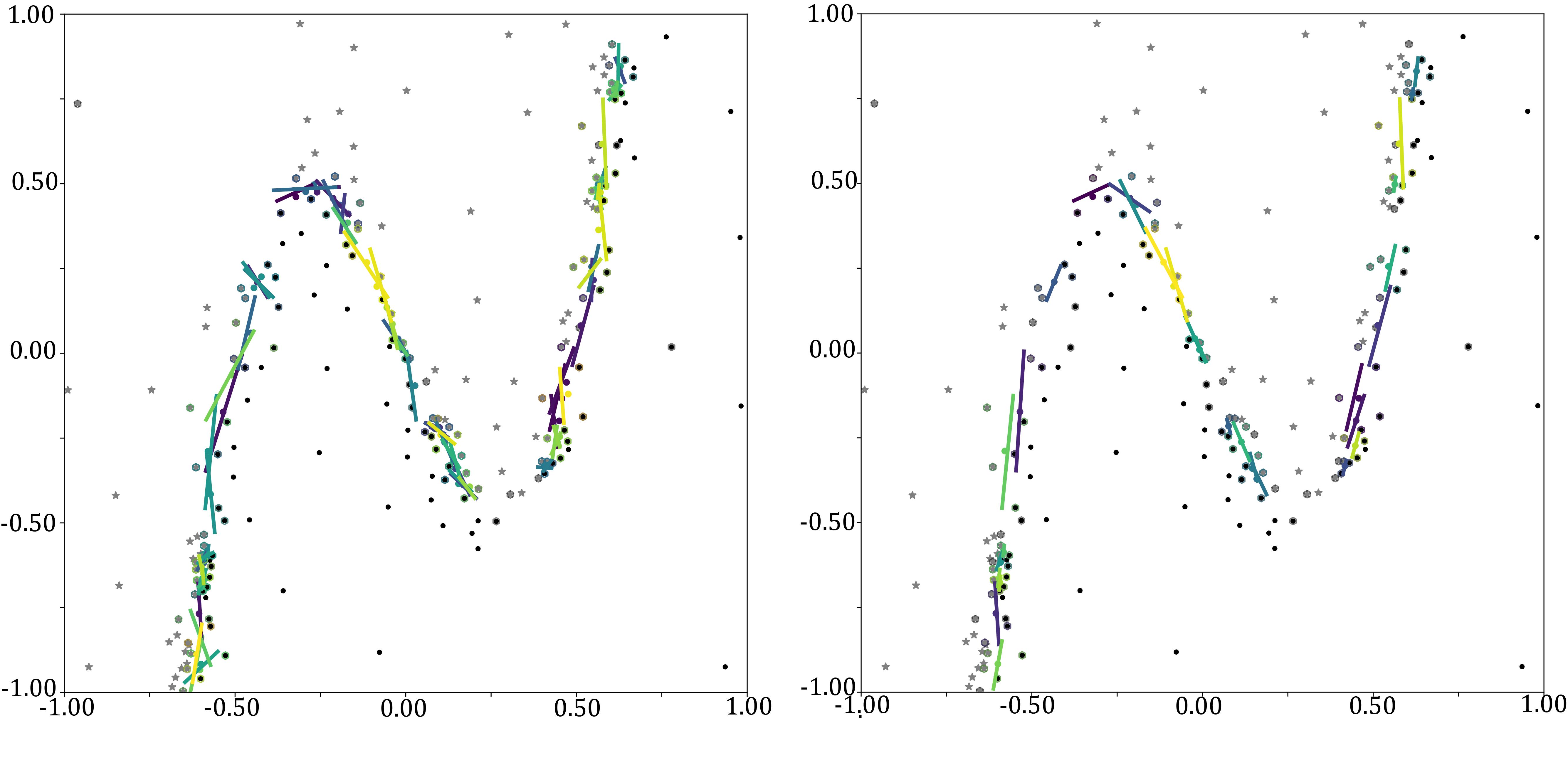}
	\caption{Classification results for the PPSVMG model with initial cluster size $K=3$ and training set with $n=200$ points. Left: $\kappa= 0$ and similarity = $1$. Right: $\kappa= 10$ and similarity = $.4$. }
	\label{fig:diagramGMSVM200}
\end{figure}

We use the Brusselator example defined in equation~\ref{Brussmodel} with a two dimensional parameter domain: 
\[
x  = (x_1, x_2, x_3)^\top \in \Lambda = [0.7,1.5]\times [2.75,3.25]\times[1.65] \subset \mathbb{R}^3,
\] 
to illustrate the effect of penalization. We set T = 5 and $q_0 = 3.75$. We use POF--Darts to draw the samples, with $\mathcal{L} = 1.5$ and initial points = 10. We show the results of varying the penalization parameter and similarity parameter in Fig.~\ref{fig:diagramGMSVM100}, \ref{fig:diagramGMSVM200}, \ref{fig:diagramGMSVM300}, and \ref{fig:diagramGMSVM500}, for varying training set sizes. As the similarity parameter increases, more clusters are combined, resulting in few linear SVM in the model. As the penalization parameter increases, the linear SVM models become more parallel to the tangent of the nonlinear decision boundary in that region.

\begin{figure}[htb]
	\centering
	\includegraphics[width=0.8\textwidth]{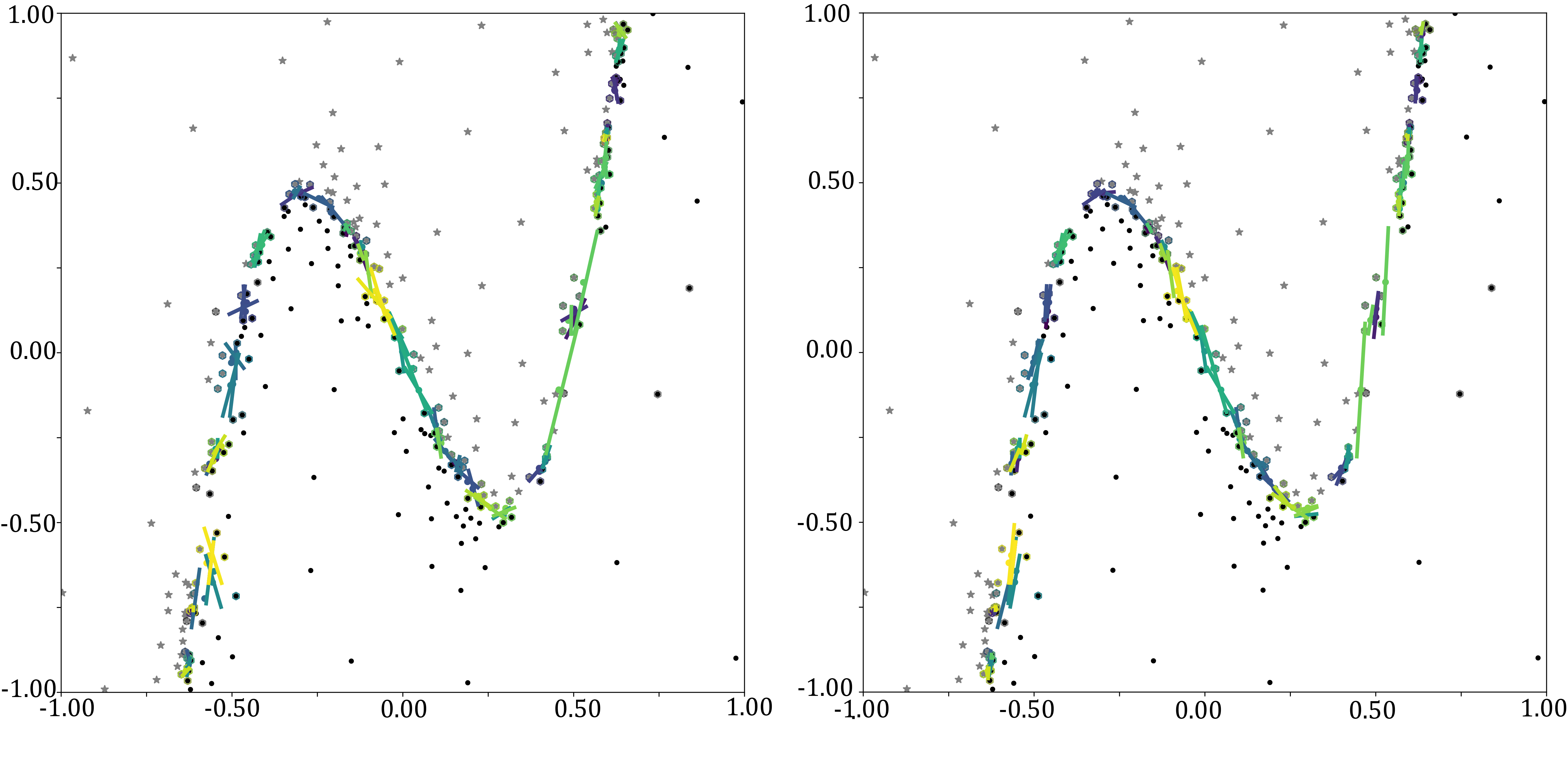}
	\caption{Classification results for the PPSVMG model with initial cluster size $K=3$ and training set with $n=300$ points. Left: $\kappa= 0$ and similarity = $1$. Right: $\kappa= 10$ and similarity = $.4$. }
	\label{fig:diagramGMSVM300}
\end{figure}

\begin{figure}[htb]
	\centering
	\includegraphics[width=0.8\textwidth]{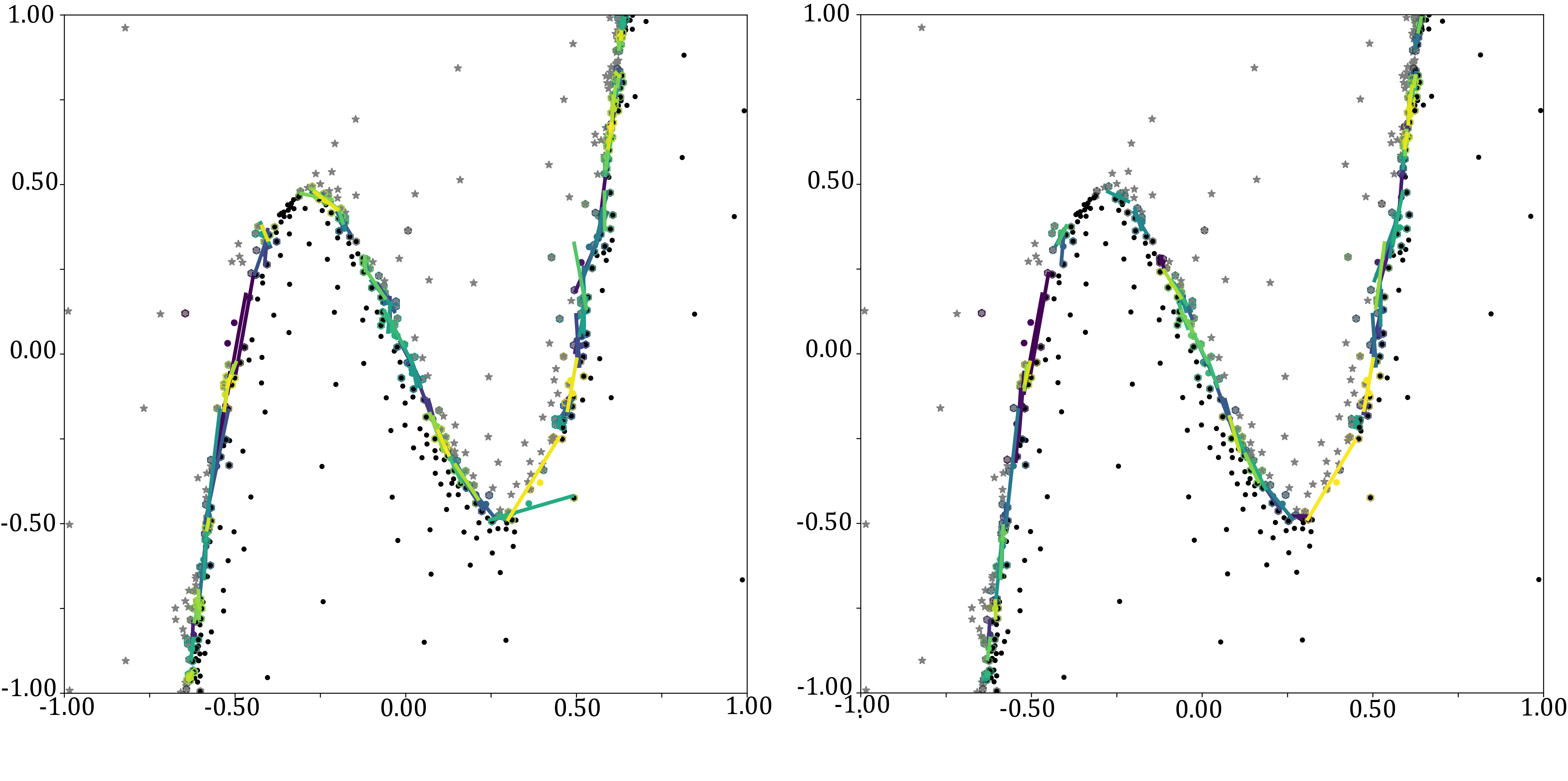}
	\caption{Classification results for the  PPSVMG model with initial cluster size $K=3$ and training set with $n=500$ points. Left: $\kappa= 0$ and similarity = $1$. Right: $\kappa= 10$ and similarity = $.4$. }
	\label{fig:diagramGMSVM500}
\end{figure}

\subsection{{Aspects of implementation}}\label{sec:implPPSVMG}

{We discuss some key aspects for using PSVMG and PPSVMG in practice.}

\subsubsection*{{Computing the gradient of $Q$}}

{The need to compute derivatives of a quantity obtained from a computer model  with respect to parameters in the model arises in sensitivity analysis, optimal control, and uncertainty quantification. Fortunately, such derivatives can be obtained with relative ease. Differentiating the process model with respect to a parameter yields a new \textit{linear} process model whose solution is the derivative of the solution of the process model with respect to the parameter. The Chain Rule can then be used to easily compute the corresponding derivative of a quantity depending on the solution of the process model. The linear process models for derivatives are generally easier and computationally cheaper to solve than the original process model. As a rule of thumb computing $Q$ and its gradient, requires significantly less than twice the computational work of computing just $Q$.}

{We use this approach to compute  $\nabla Q$ in this paper. The derivatives were computed by hand because the models are simple. Details are supplied in \cite{ZhuXuankang2022Elaf}. Fortunately, the tedious work of computing derivatives of a computer code with respect to parameters can be automated by \textit{Automatic Differentiation (AD)} methodology, which is implemented in a number of public software packages that can be applied to a computer code for a process model \cite{Neidinger2010}. }

\subsubsection*{{Tuning parameters}}

 {We initiate POF--Darts sampling with $10$ uniformly distributed random samples and choose $\mathcal{L} = 1.5$. Experimentally, this gives a good balance between underestimating the distance of the starting sample points to the nonlinear decision boundary and the number of iterations required to obtain training points close to the boundary.  We also experimented with including  $\mathcal{L}$ in the cross-validation optimization with the other hyperparameters, but this did not  improve results.  We optimize the other hyperparameters in PPSVMG using cross--validation to achieve the highest average training accuracy. For the cross--validation, the training set is divided into $10$ ``folds.'' Treating each fold as a validation set, the model is trained on the points in the other folds  and tested on the validation set. This is repeated using each fold as a validation set exactly once. Below, we use training sizes of 40, 60, 80, 100, 120, 150, and 200, with a testing size of 2000. }

\subsubsection*{{Execution time}}
{The great majority of the execution time for the examples in this paper was used for cross--validation selection of hyperparameters  because this involves solving multiple nonconvex optimization problems. The execution time for a single evaluation of PSVMG and PPSVMG for fixed hyperparameter values is $O(kd)$, where $k$ is the number of clusters and $d$ is the dimension.}

\section{Convergence properties of the GES}\label{sec:proof}
We prove two  convergence results for the Gabriel neighbors. We begin by showing that the Gabriel neighbors approach the nonlinear decision boundary as the number of samples increases.

Recall that $\mathcal{G}_n$ denotes the Gabriel Edited Set for the training set $\{\mathcal{X}_n,\mathcal{Y}_n\}=\{x_i,y_i\}_{i=1}^n$ with $n$ sample points and $\mathcal{N}_n$ denote the indices of Gabriel neighbors. We set $M_n = \max_{(i,j)\in\mathcal{N}_n}{\|x_i - x_j \|}$ denote the maximum Euclidean distance between  pairs of Gabriel neighbors in $\mathcal{G}_n$.

\begin{theorem}\label{thm:Gneigh}
	 For any $\epsilon> 0$, 
\[
\lim_{n\to \infty} P(|M_n| > \epsilon) = 0,
\]
where $P$ is the uniform probability measure on $\Lambda$.
\end{theorem}

This result implies that the Gabriel neighbors converge to the nonlinear decision boundary as the number of samples in the training set increase. 

\begin{proof} 
	% Denote \(A_t\) as the area where \(|f(x) - f(c)<t|\), \(A_t\) is the area near the boundary where the function is close to the threshold. Denote \(G_j\) as the uncovered region in the support after j samples. Let j be the number of samples we have generated, and let i be the number in \(A_t\). Suppose that the uncovered area is large compared to $A_t (0< k < 1)$, then the convergence rate of the uncovered region by POF--Darts is~\cite{POF-darts}: \[
	% P(G_j \ge kA_t) < \binom{j}{i}k^{-j + i}.
	% \]
	% \par Suppose \((x_{i1},x_{i2})\) are the Gabriel neighbors pair corresponds with \(M_j\). The smallest closed sphere $S_j$ containing \(x_{i1}\) and \(x_{i2}\) has radius \(r_j = \frac{M_j}{2}\). By the definition of GES, $S_j$ does not contain any sample points other than \(x_{i1}\) and \(x_{i2}\), so volumn of $S_j = \frac{\pi^{n/2}}{\frac{n}{2}!}M_j^n$ hence \[
	%     P(|M_{j+k}| \ge \epsilon | M_k = \epsilon) \le \left(1 - \frac{\frac{\pi^{n/2}}{\frac{n}{2}!}\epsilon^n}{\mu(\Lambda)}\right)^j \text{ for sufficiently large j.}
	% \]
	% Thus, it follows that\[
	% \lim_{j\to \infty} P(|M_j|\ge \epsilon) =  0.
	% \]
	See Fig.~\ref{fig:proofboundary}. 	Let $S\subset \Lambda$ be a fixed sphere and  define 	$X_n = \sum_{i = 1}^{n} \mathbb{I}(x_i \in S)$, where $\mathbb{I}$ is the  indicator function. By the Central Limit Theorem (\cite{PT_book}), 
\[
X_n \xrightarrow{d} N\left(n\frac{\mu(S_n)}{\mu(\Lambda)}, n\,\frac{\mu(S_n)}{\mu(\Lambda)}\,\left( 1- \frac{\mu(S_n)}{\mu(\Lambda)}\right) \right),
\]	
where $N(x,\Sigma)$ is the normal distribution with mean $x$ and covariance $\Sigma$ and $\mu$ is the Lebesgue measure on $\Lambda$. 	Thus, 
	\begin{equation}\label{normcong}
		\lim_{n \to \infty }P(X_n = 0 ) =  \lim_{n\to \infty} \Phi\left( \frac{-n\frac{\mu(S_n)}{\mu(\Lambda)}}{\sqrt{n\,\frac{\mu(S_n)}{\mu(\Lambda)}\,\left( 1- \frac{\mu(S_n)}{\mu(\Lambda)}\right)} }  \right) = 0,
	\end{equation}
where $\Phi$ is the inverse of the cumulative distribution function for the normal distribution. 

Fix $\epsilon >0$. We have $Q^{-1}(q_0)\subset \bigcup_{z \in Q^{-1}(q_0)} S(z,\epsilon)$, where $S(a,r)$ is the open sphere centered at $a$ of radius $r$. Since $Q^{-1}(q_0)$ is compact, there are a finite set of $z_i \in Q^{-1}(q_0)$, $1 \leq i \leq m$, with $Q^{-1}(q_0)\subset \bigcup_{i=1}^m S(z_i,\epsilon)$. %Let $S=\bigcup_{i=1}^m S(z_i,\epsilon)$ and define $X_n = \sum_{i = 1}^{n} \mathbb{I}(x_i \in S)$. From (\ref{normcong}), it follows that  $\lim_{n \to \infty }P(X_n = 0 )=0$. 
Because $Q^{-1}(q_0)$ is compact and the $S(z_i,\epsilon)$ are open, the intersections of $S(z_i,\epsilon)$ with regions $\{x:Q(x)\geq q_0\}$ and $\{x:Q(x)< q_0\}$ are each nonempty and open for $1\leq i\leq m$. 

For each $1 \leq i \leq m$, there exists a $0 < \eta < \epsilon$ and points $w_1 \in S(z_i,\epsilon) \cap \{x:Q(x)\geq q_0\}$ and $w_2 \in S(z_i,\epsilon) \cap \{x:Q(x)> q_0\}$, such that $S(w_1,\eta)\subset S(z_i,\epsilon) \cap \{x:Q(x)\geq q_0\}$ and $S(w_2,\eta)\subset S(z_i,\epsilon) \cap \{x:Q(x)> q_0\}$.  If $X^1_n = \sum_{i = 1}^{n} \mathbb{I}(x_i \in S(w_1,\eta))$ and $X^2_n = \sum_{i = 1}^{n} \mathbb{I}(x_i \in S(w_2,\eta))$, (\ref{normcong}) implies that $\lim_{n \to \infty }P(X^1_n = 0 )=0$ and $\lim_{n \to \infty }P(X^2_n = 0 )=0$. It follows that $\lim_{n\to \infty} P(M_n \leq \epsilon) = 1$. 
\end{proof}

\begin{figure}[htb]
	\centering
	\includegraphics[width= 10cm]{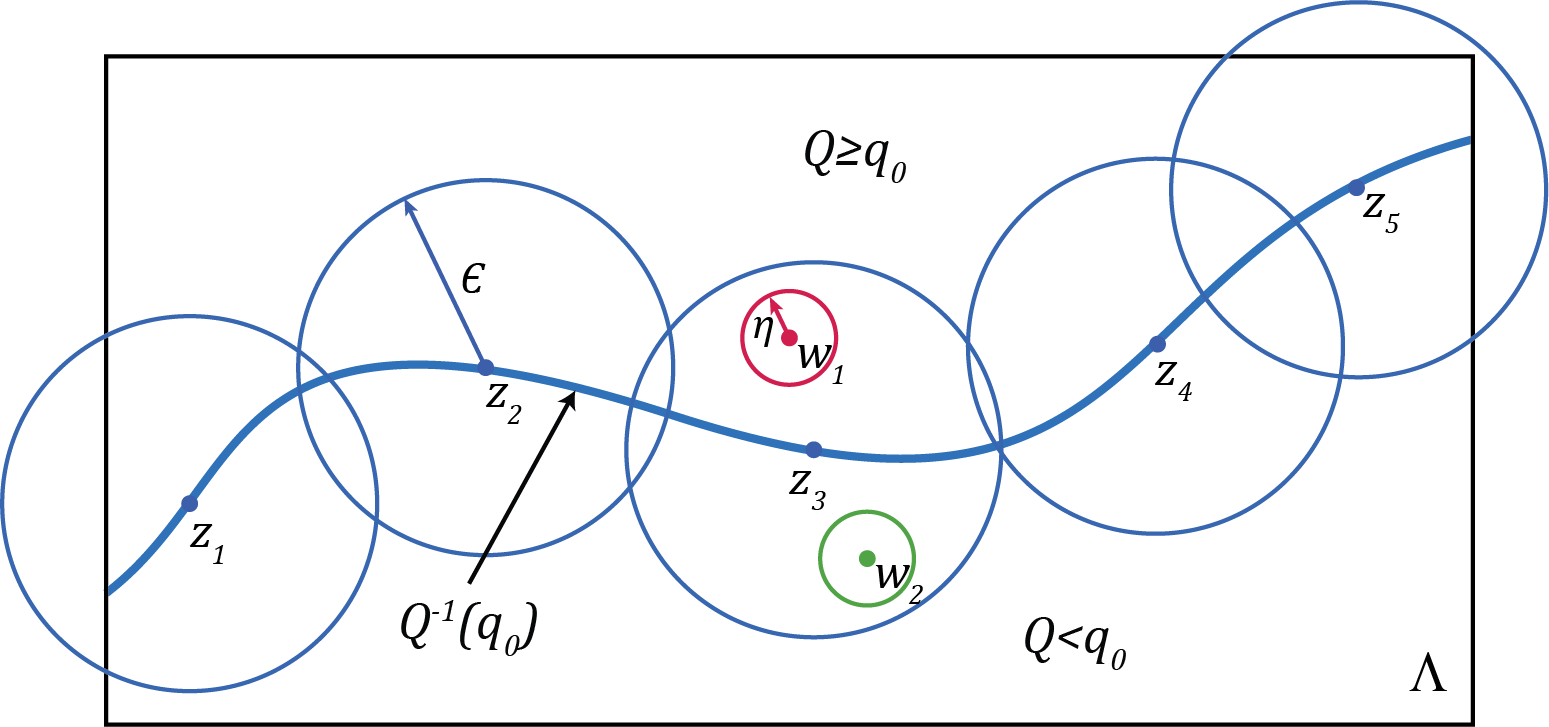}
	\caption{Illustration for the proof of Theorem~\ref{thm:Gneigh}}
	\label{fig:proofboundary}
\end{figure}

For the next result, let $\mathcal{E}_n$ denote the set of intersections $\widehat{x}_{ij}^n$ of the Gabriel graph of $x_i$ and $x_j$ for $(i,j)\in\mathcal{N}_n$ with $Q^{-1}(q_0)$. For $z \in Q^{-1}(q_0)$, let $\widehat{x}_{ij,z}^n$ denote the nearest neighbor of $z$ and points in  $\mathcal{E}_n$. 

\begin{theorem}\label{thm:GESinter}
	For any $\epsilon > 0$, 
	\[
	\lim_{n\to \infty}P(\|\widehat{x}_{ij,z}^n - z\|>\epsilon) = 0.
	\]
\end{theorem}

\begin{proof}
	
	Fix $z \in Q^{-1}(q_0)$ and $\epsilon> 0$. Let $S$ denote the sphere centered at $z$ with radius $\epsilon$. The proof of Theorem~\ref{thm:Gneigh} implies that with probability $1$, there are points $x_i, x_j \in \mathcal{S}$ with $y_i \neq y_j$ for all $n$ sufficiently large. It follows that with probability $1$, $\|\widehat{x}_{ij,z}^n - z\|<\epsilon$  for all $n$ sufficiently large. The result follows. 

\end{proof}

This result implies that the centers of the clusters used to construct the localized linear SVM models in the PPSVMG model become increasingly closer to the nonlinear decision boundary as the size of the training set increases. If we fix the cluster size $k$, a modification of the proof of Theorem~\ref{thm:Gneigh} implies that the points in the clusters that comprise the PPSVMG model also approach the nonlinear decision boundary and the points in the clusters become closer to each other as the size of the training set increases.

Together, Theorems~\ref{thm:Gneigh} and \ref{thm:GESinter} imply that the local linear SVM models comprising the PPSVMG model become closer to the nonlinear decision boundary as the initial sample size increases. Since the predicted classification of an arbitrary point is based on an ensemble average of predictions of the different models and we do not require that the linear models become parallel to the tangent plane of the nonlinear decision boundary, it is  difficult to analyze the prediction accuracy of the PPSVMG model as the sample size increases. However, numerical investigation, for example as seen in  Figures~\ref{fig:diagramGMSVM100}, \ref{fig:diagramGMSVM200}, \ref{fig:diagramGMSVM300}, and \ref{fig:diagramGMSVM500}, suggests that the linear SVM models comprising the PPSVMG model  converge to tangent planes of the nonlinear decision boundary.

\subsection{{Convergence rates}}

{Generally, it is hard to establish meaningful asymptotic convergence rates for estimates computed using adaptive sampling techniques like POF--Darts. Issues include the use of a feedback mechanism in sampling means properties like independence do not hold and the adaptive process involves additional computational expense. In the limit of very large numbers of samples, adaptive and non-adaptive sampling methodologies have similar asymptotic properties. The largest benefit of adaptive sampling is expected for small numbers of samples.}

{In the tests below, we consider performance through the ranges of small to large numbers of samples.}

\section{Experimental investigation}\label{sec:simul}
In this section, we investigate the performance of the PPSVMG model on several test problems.

\subsection{Model benchmark and validation}\label{compmethod}
We compare the performance of the PPSVMG model, measured as prediction accuracy versus the size of the training set, against several common classification methods given in Table~\ref{tab:methods_defination2}.  Operational details for these methods and the software packages are provided in Appendix~\ref{app:hyper}.

\begin{table}[htb]
    \centering
         \begin{tabular}{ |p{10.5cm}||p{2.5cm}| }
         \hline
         Method & Abbreviation\\
         \hline
         K nearest neighbors& KNN \\
         Random forest   & random\_forest \\
         XGBoost &   XGBoost  \\
         Multilayer perception & MLP  \\
         Kernel SVM (SVM$\_$kernel)    & SVM$\_$kernel \\
         Profile SVM & PSVM \\
         SVM\_KNN & SVM\_KNN\\
         Penalized Profile SVM on GES with clusters centered on CBPs& PPSVMG\\
         Profile SVM on GES with K-means clustering& PSVMG\\
         \hline
        \end{tabular}
    \caption{Abbreviations of the classification methods used in the comparison.}
    \label{tab:methods_defination2}
\end{table}

To conduct the comparison, we use the best performance for each model, obtained by determining the hyperparameters that achieve the highest average training accuracy using cross--validation. As mentioned, we use training sizes of 40, 60, 80, 100, 120, 150, and 200, with a testing size of 2000. {For samples generated by POF--Darts sampling, we set the initial samples as 10, the scale factor $\mathcal{L} = 1.5.$} To reduce random effects in presenting the results, we repeat the sampling and training process 20 times for each training size. 

For the cross--validation, the data set is divided into $10$ subsets or "folds." The model is trained on a subset of the data (defining the training set) and tested on the remaining portion (the validation set). This process is repeated several times, with each fold serving as the validation set exactly once. The results from each iteration are averaged to provide an overall performance metric, helping to ensure that the model's performance is not overly dependent on any particular subset of the data. Using cross--validation  reduces the risk of overfitting and yields a more accurate estimate of the model's performance on unseen data.  

\subsection{Test problems}
We consider four nonlinear decision boundaries.

\subsubsection*{Composed function 1}

The boundary is determined by the function,
\begin{equation}
    Q_1(x_1, x_2) = \left(x_2 - 0.5 \cdot \big(\tanh(20x_1) \cdot \tanh\left(20(x_1 - 0.5)\big) + 1\right) \cdot e^{0.2x_1^2}\right)^2,
\end{equation} 
where $(x_1, x_2) \in \Lambda = [0,2] \times [0,2]$ and the threshold value is $q_0=.5$. We plot the classification of $\Lambda$ in Fig.~\ref{fig:simulfun} (a). {The estimated probability of failure is $ 0.4562$ with standard deviation $0.00664$ computed with $5000$ points and $20$ repeats.}

\begin{figure}[htb]
	\centering
	\includegraphics[width= 6cm]{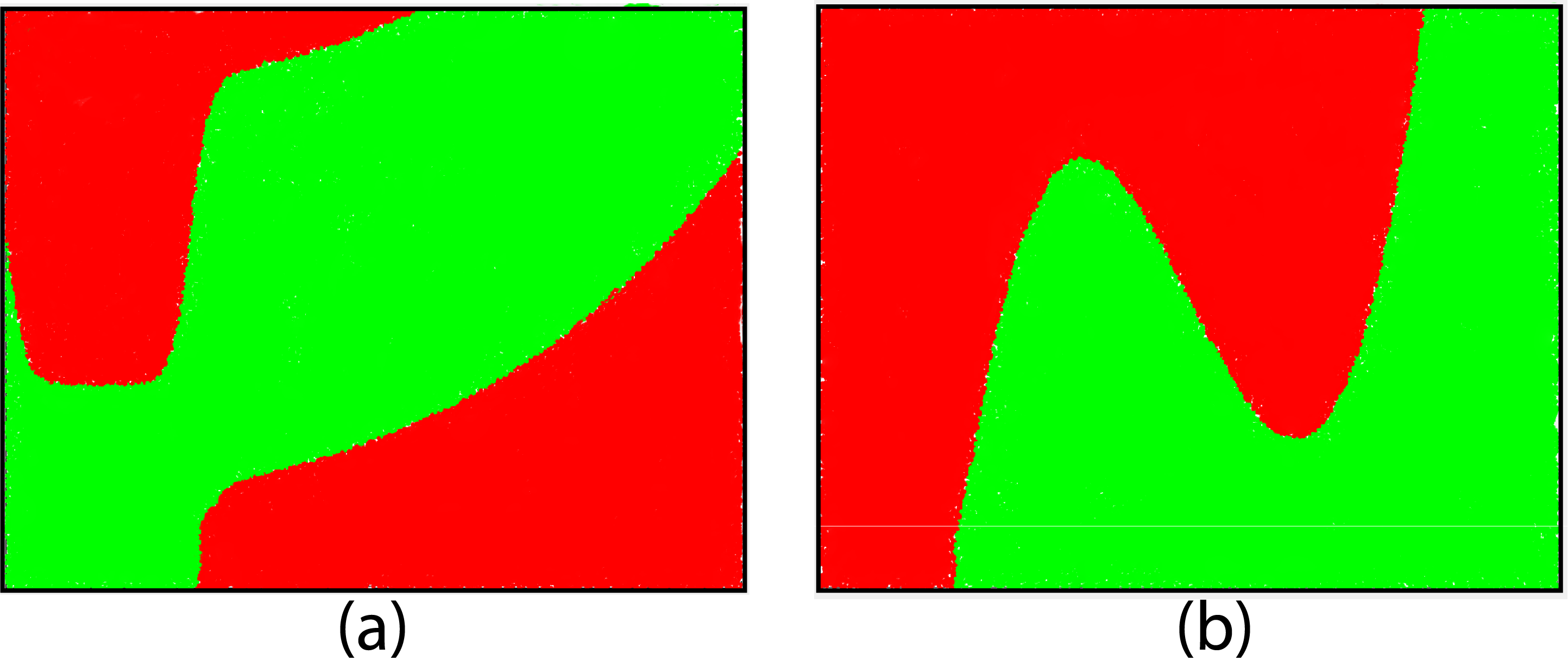}
	\caption{Classification of $\Lambda$ by the nonlinear decision boundaries corresponding to (a) $Q_1$ and (b) $Q_2$.}
	\label{fig:simulfun}
\end{figure}

$Q_1$ is constructed to create a nonlinear decision boundary for which the curvature varies significantly and to have a separated region for one of the classifications. Additionally, the magnitude of $\nabla Q_1$ is very large in the region of the nonlinear decision boundary, signifying that $Q$ changes rapidly in the region of the boundary.

\subsubsection*{Composed function 2}

The boundary is determined by the function,
\begin{equation}
	Q_2(x_1, x_2) = 1 + \tanh\left(10 \cdot \left(x_2 - 10x_1(x_1 - \frac{1}{2})(x_1 + \frac{1}{2})\right)\right),
\end{equation}
where $(x_1, x_2) \in \Lambda=[-1,1] \times [-1,1]$.  We plot the classification of $\Lambda$ in Fig.~\ref{fig:simulfun} (b). {The estimated probability of failure is $ 0.5001$ with standard deviation $0.00483$ computed with $5000$ points and $20$ repeats.}

$Q_2$ is constructed to create a nonlinear decision boundary for which the curvature and convexity varies significantly. Additionally, the magnitude of $\nabla Q_1$ is very large in the region of the nonlinear decision boundary, signifying that $Q$ changes rapidly in the region of the boundary.

\begin{figure}[htb]
	\centering
	\includegraphics[width =.8\textwidth]{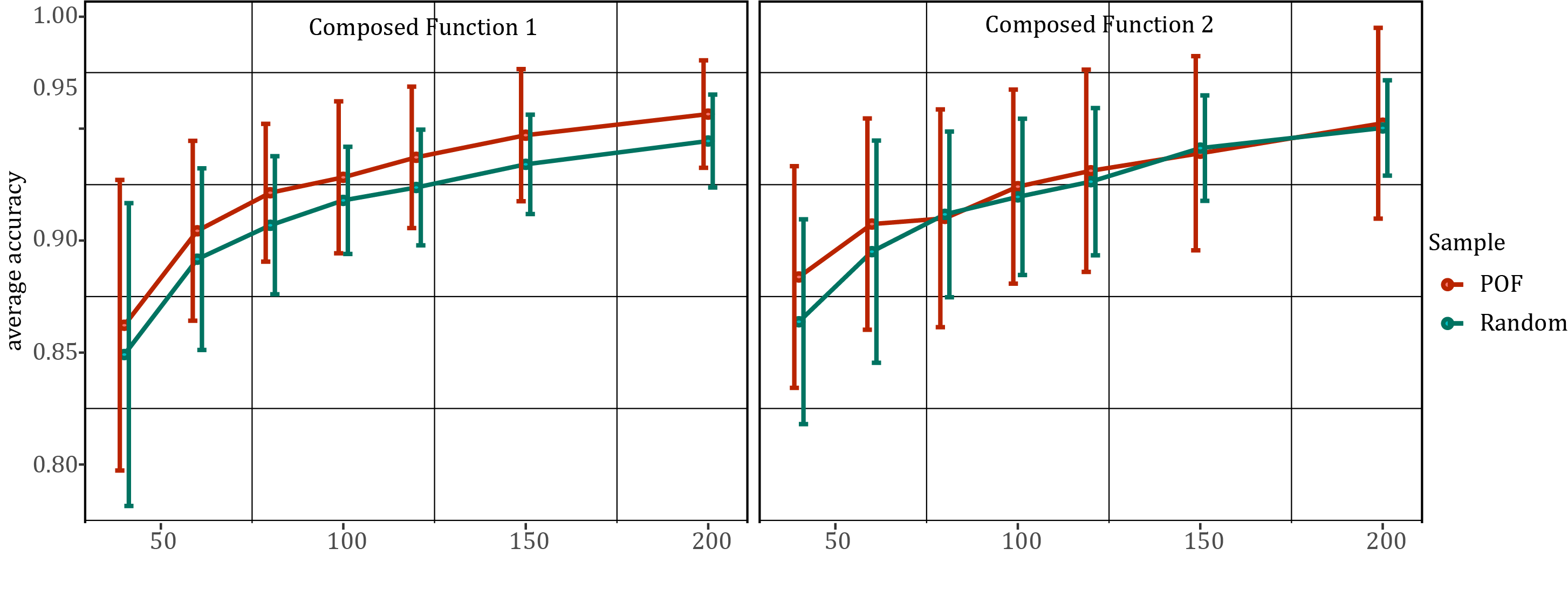}
	\caption{Performance for POF--Darts and random sampling averaged over all models for the composed functions 1 (left) and 2 (right).}
	\label{fig:enter-label1}
\end{figure}

\subsubsection*{The Brusselator model}

We consider the Brusselator model (\ref{Brussmodel}) with physical characteristics equal to the two rate parameters and the initial value for $z_1$. So, $x  = (x_1\; x_2\; x_3)^\top \in \Lambda = [0.7,1.5]\times [2.75,3.25]\times[1,2] \subset \mathbb{R}^3$ and 
$
Q(x) = \frac{1}{5}\int_{0}^{5}(z_1(t,x) + z_2(t,x))dt$. We present details of the numerical solution of the model in Appendix~\ref{app:nummethod}. The threshold value is $q_0=3.75$. {The estimated probability of failure is $0.1156$ with standard deviation $0.00246$ computed with $5000$ points and $20$ repeats.}

\subsection*{An elliptic differential equation}
We consider boundary defined by the solution of the elliptic two point boundary value problem,
\begin{equation}\label{ellmodel}
\begin{cases}
    -\frac{d}{ds}((s^2e^{-x_1 s} + 0.05)\frac{dz}{ds})+x_2 \frac{dz}{ds} = (1 - s)\tanh( 4(s-x_3) ) + \sin(5\pi x_4 s), & 0< s < 1, \\
    z(0) = z(1) = 0.
\end{cases}
\end{equation}
where $ \Lambda = [1,5]\times[0.1,0.3]\times[0,1]\times[0,2]$ and $Q$ is the average spatial value:
\[
Q(x) = \int_{0}^{1} u(s)\, ds.
\]
The threshold value is $q_0 = 0$. {The estimated probability of failure is $0.5428$ with standard deviation $ 0.00569$ computed with $5000$ points and $20$ repeats.}

\begin{figure}[htb]
	\centering
	\includegraphics[width =.8\textwidth]{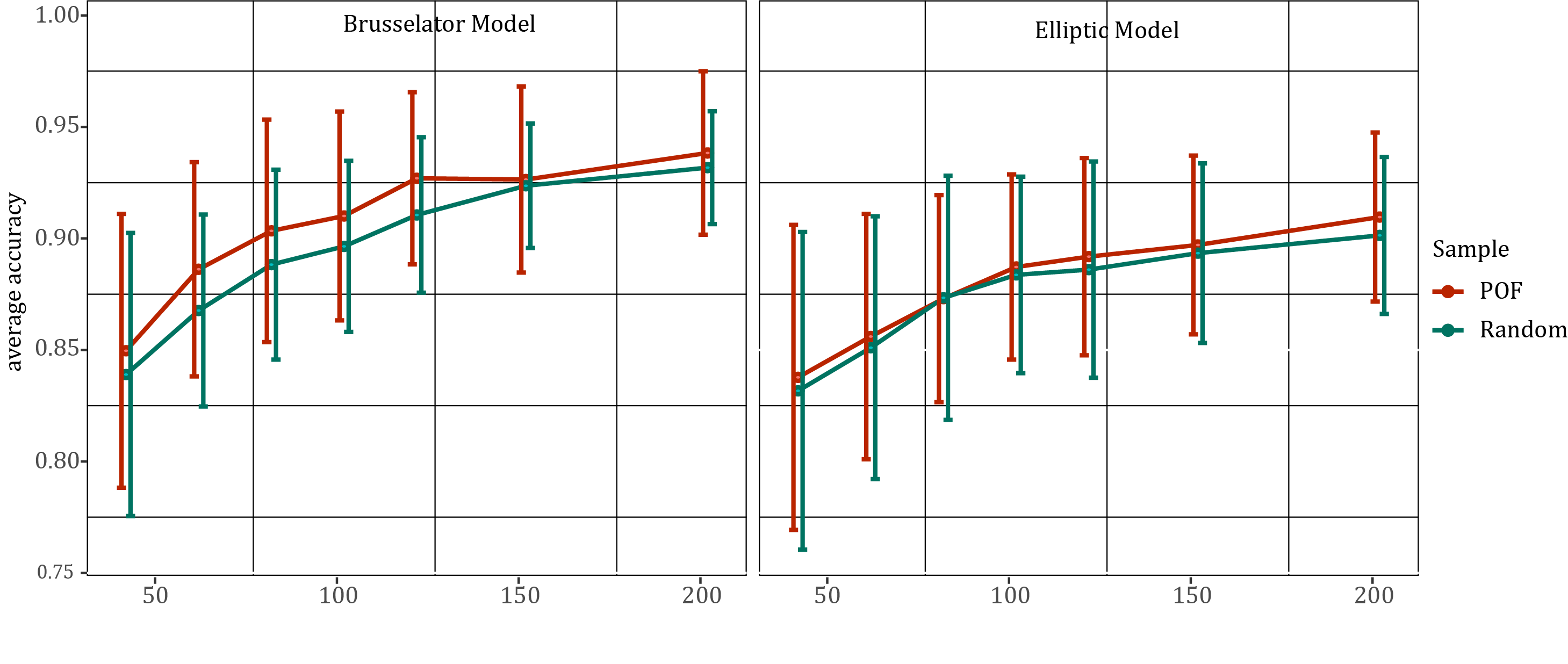}
	\caption{Performance for POF--Darts and random sampling averaged over all models for the Brusselator (left) and elliptic (right) problems.}
	\label{fig:enter-label2}
\end{figure}

\subsection{Comparison of POF--Darts and random sampling}

\begin{figure}[htb]
	\centering
	\includegraphics[width=.8\textwidth]{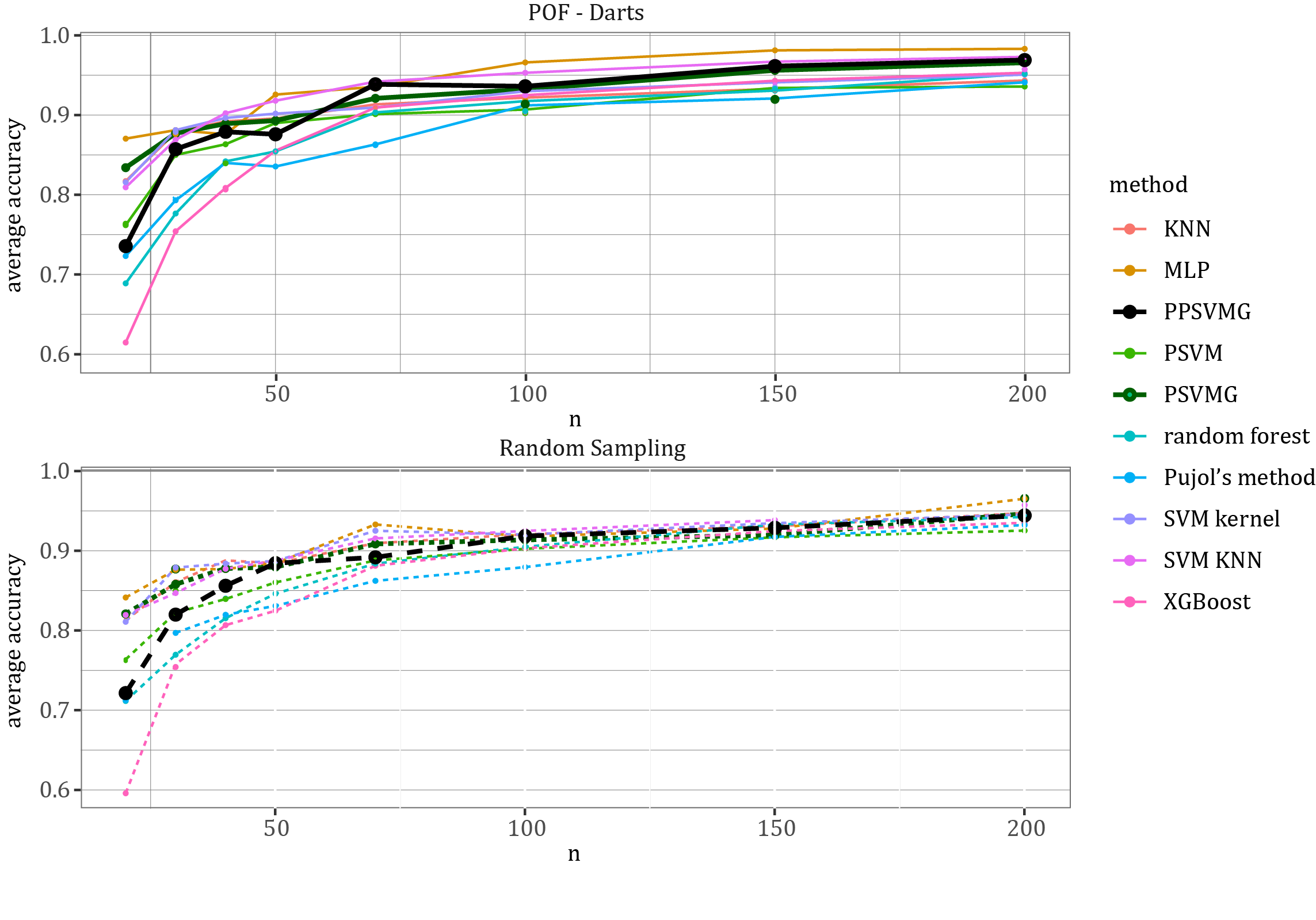}
	\caption{Averaged best model accuracy for composed function 1.}
	\label{fig:enter-label3}
\end{figure}

\begin{figure}[htb]
	\centering
	\includegraphics[width=.8\textwidth]{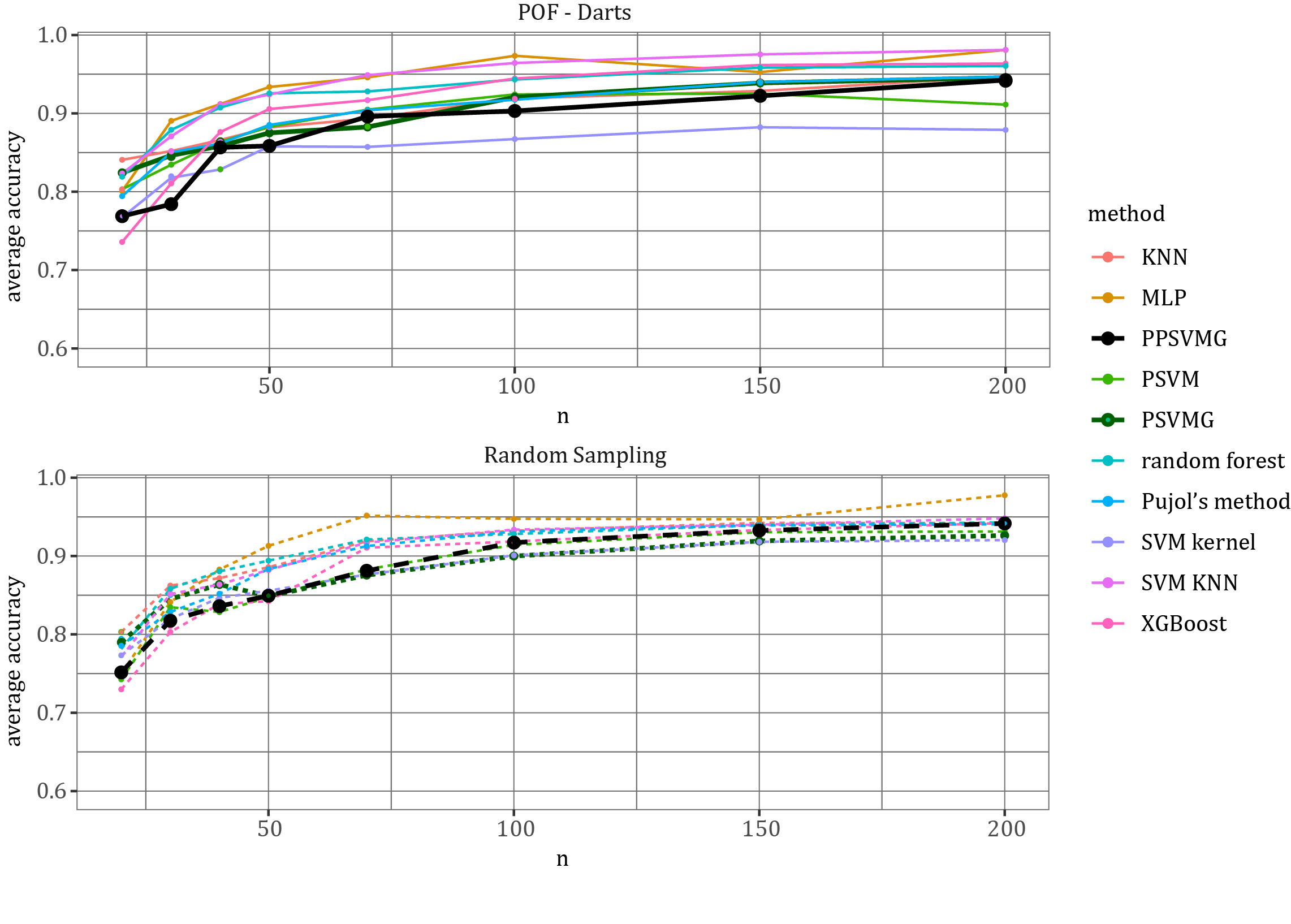}
	\caption{Averaged best model accuracy for composed function 2.}
	\label{fig:enter-label4}
\end{figure}

\begin{figure}[htb]
	\centering
	\includegraphics[width=.8\textwidth]{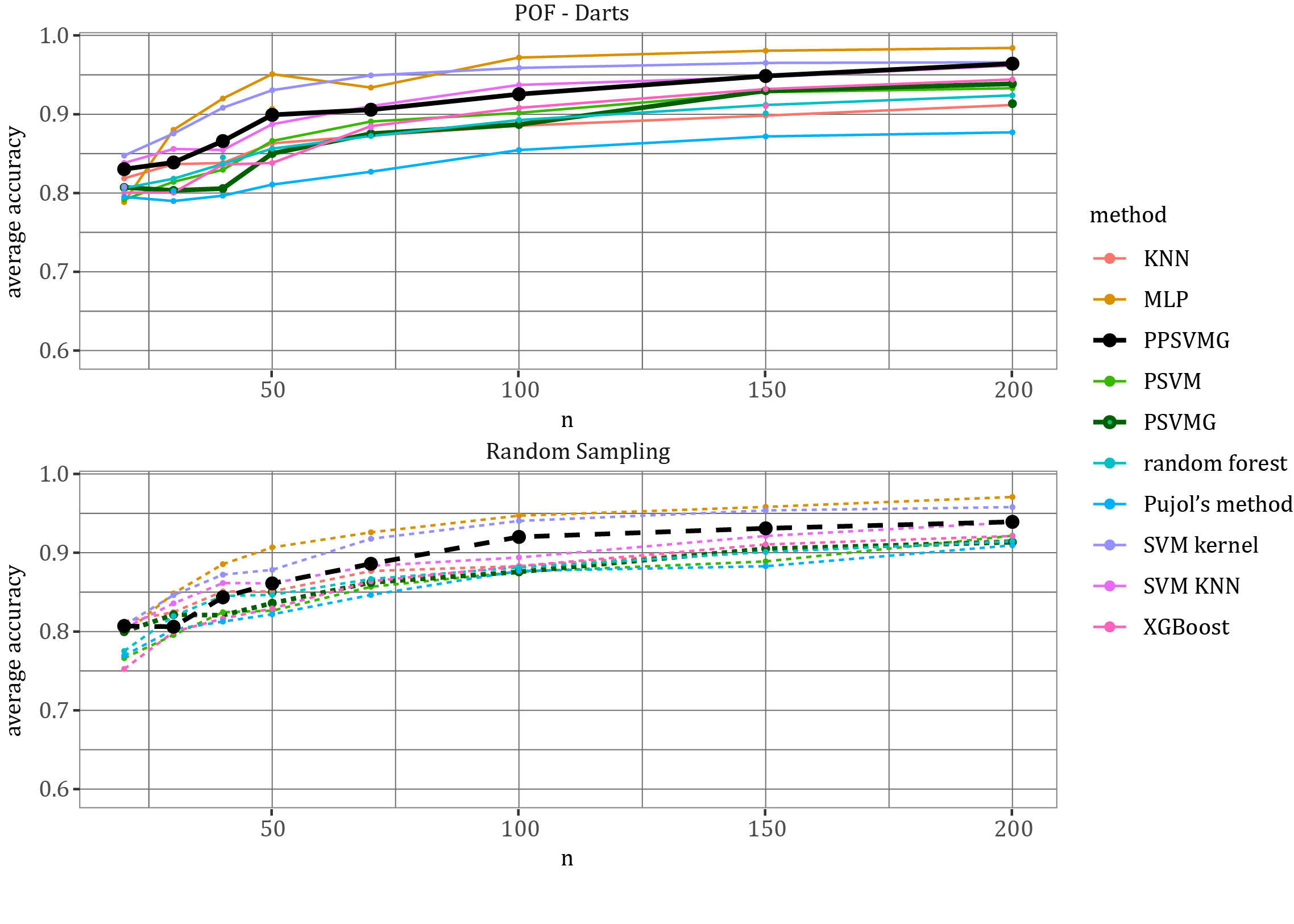}
	\caption{Averaged best model accuracy for the Brusselator model.}
	\label{fig:enter-label5}
\end{figure}

\begin{figure}[htb]
	\centering
	\includegraphics[width=.8\textwidth]{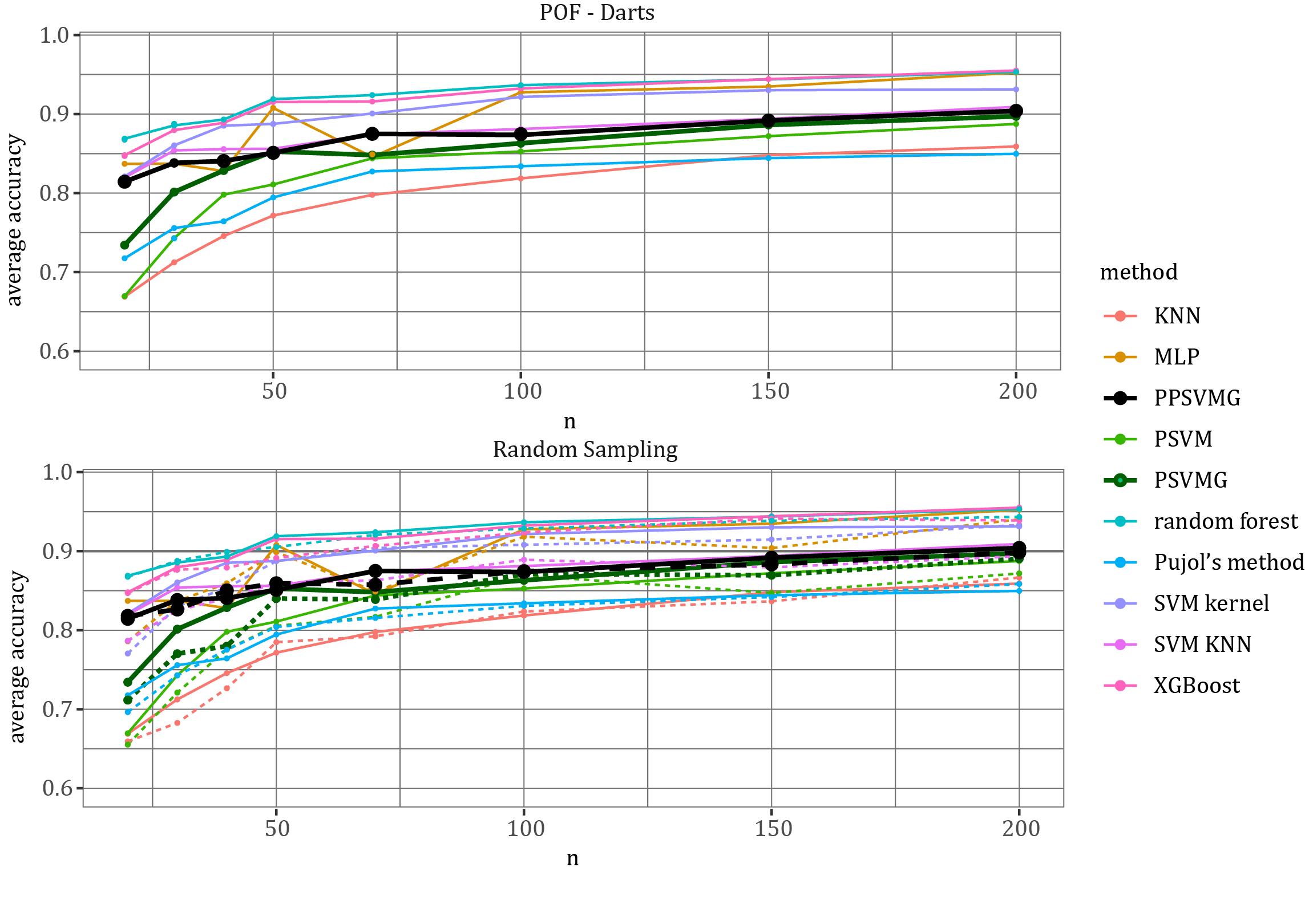}
	\caption{Averaged best model accuracy for elliptic differential equation.}
	\label{fig:enter-label6}
\end{figure}

We compare the average performance over all models using training sets constructed from POF--Darts with $1-d$ darts and random sampling using a uniform probability distribution in Figures \ref{fig:enter-label1} and \ref{fig:enter-label2}. {The overall average model performance is determined by computing the mean of  average benchmark model accuracies across the 20 repetitions. Additionally, we include the 95$\%$ interval to provide information about variance.} The models computed using POF--Darts generated samples almost always yields higher accuracy, and often with a significant improvement for small training sets. The improvement is most distinct for the composed function 1 and the Brusselator test problems.

\subsection{Comparison of models}

We compare the performance of the models listed in Sec.~\ref{compmethod} using training sets constructed from POF--Darts with $1-d$ darts and random sampling using a uniform probability distribution in Figures \ref{fig:enter-label3},  \ref{fig:enter-label4}, \ref{fig:enter-label5}, and \ref{fig:enter-label6}. {Solid lines represent the average model performance across 20 training sets generated by POF--Darts, while dashed lines indicate the average model performance across 20 training sets obtained through random uniform sampling.}

The PPSVMG  method demonstrates significantly better performance than PSVMG as well as PSVM, Pujol's method, and SVM KNN across these four examples. Additionally, PPSVMG ranks among the top 2–3 models in the Brusselator, Function 1, and Function 2 examples in terms of accuracy and performs above average in the elliptic example.

\section{{Application to the Lotka--Volterra model}}\label{sec:LVM}

{The Lotka–-Volterra model is  population dynamics model for multiple species competing in an environment of limited resources  that has found applications in a range of situations, e.g., animal populations, food webs, bacteria populations, and business economics (\cite{Rescigno1977,krik1979,Freed1984,Hsu2015,Gav2018,Butler_2018}). The model consists of a system of nonlinear ordinary differential equations. We focus on a 3-species model,
\begin{equation}
\label{eqn:lotka}
\begin{cases}
\displaystyle \frac{dz_i}{dt} = r_i\, z_i\, \bigg(1 - \sum_{j=1}^3 s_{ij}\,z_j\bigg), & t \in (0,10], \\
z_i(0) = z_{i,0},
\end{cases} \quad \text{ for } i = 1, 2, 3,
\end{equation}
with \textit{interaction} parameters $\{s_{ij}\}_{i,j=1}^3$, \textit{reproduction} parameters $\{r_i\}_{i=1}^3$, and initial conditions $\{z_{i,0}\}_{i=1}^3$. The matrix $s$ may be asymmetrical, e.g., when one species  strongly suppresses another but the opposite is not true. The Lotka--Volterra model is used to address a fundamental question for competing species: Do the species coexist with populations in a form of equilibrium or does one or more species die out?}

{ We fix  $z_{i,0} = (10,5,2)$, and the self--interacting parameters $s_{ii} = .5$ for $i = 1,2, 3$. Thus, the parameter space is 9 dimensional. We define the sample space $\Lambda$ with $s_{ij} \in [0.25,0.75]$ and $r_i \in [0.1,2]$ for $i, j = 1,2,3$, $i \ne j$. The quantity of interest is the population of species 3, $z_3(10)$, at time $T = 10$, and we compute the probability of  success $P(z_3(10) > 0.8)$, which is the probability that the population of species $3$ is remains higher than $.8$ through time $10$.}

{Figure~\ref{fig:LotVol-sample} shows $20$ trajectories of $z_3$ on $[0,10]$ corresponding to $20$ random samples drawn from the uniform distribution on $\Lambda$.  We compute numerical solutions of \ref{eqn:lotka} using Huen's method (\cite{Ascher1998}) with 1000 time steps. We see that this parameter domain leads to decaying population for species $3$ at various rates.  We also plot an empirical probability density function for the final values $z_3(10)$ for $5000$ random samples from a uniform distribution on $\Lambda$.}
\begin{figure}[htbp]
	\centering
	\includegraphics[width = .8\textwidth]{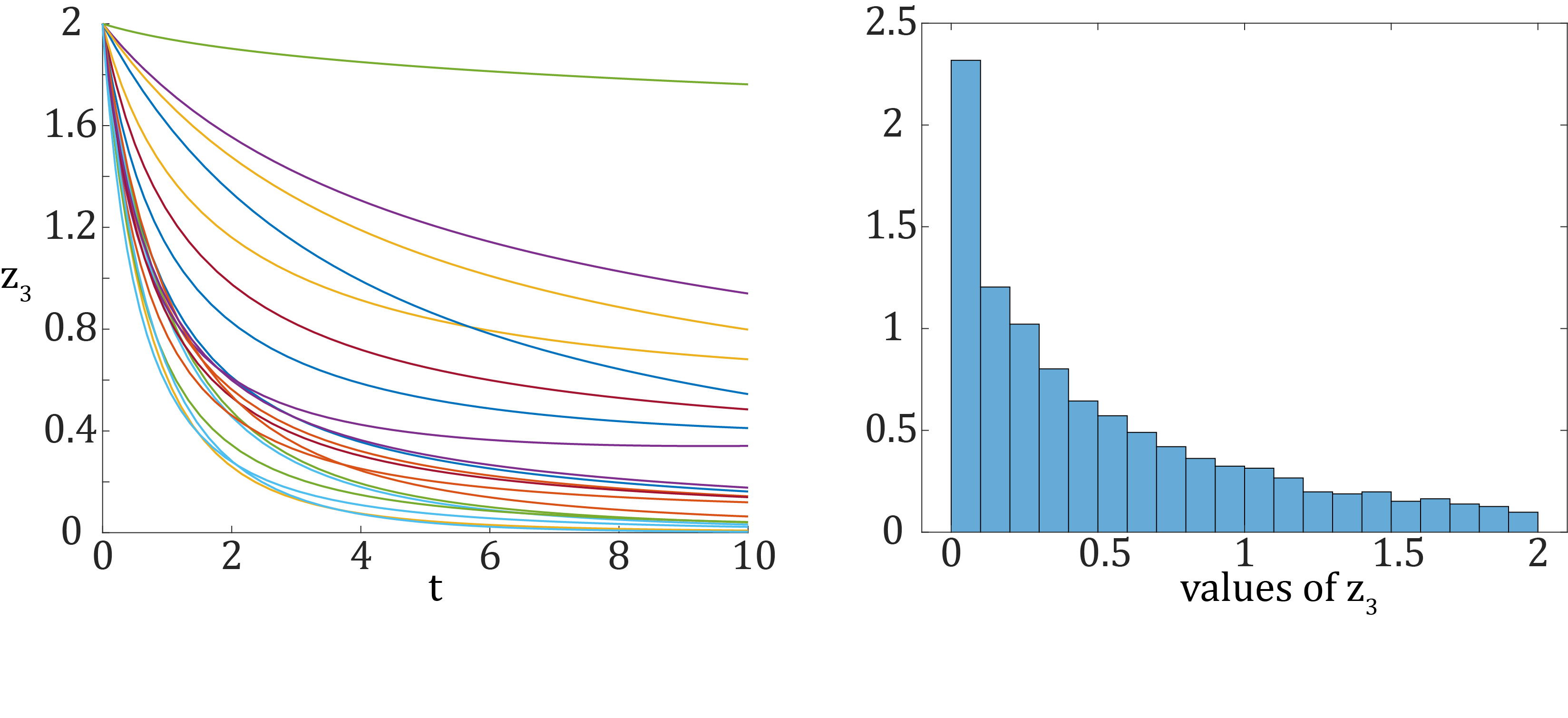}
	\caption{{Left: $20$ trajectories of $z_3$ on $[0,10]$ corresponding to $20$ random samples of the parameters. Right: Empirical probability density function for  $z_3(10)$ for $5000$ random samples of the parameters.}}
	\label{fig:LotVol-sample}
\end{figure}

{We compute a direct Monte Carlo method estimate for the probability of success by choosing a random sample using a uniform distribution in $\Lambda$ and calculating the proportion of points for which \( z_3(10) > 0.8 \). We compare the results to the estimate provided by PPSVMG  using sample sizes of,
\[
n \in \{20, 30, 40, 50, 70, 100, 150, 200, 300, 500, 750, 1000\}.
\]
For each sample size $n$, we perform 5-fold cross--validation to tune the PPSVMG model and then generate 5,000 surrogate predictions for Monte Carlo estimation.  Each experiment is repeated 30 times to estimate variance. For stability, we remove the two most extreme estimations for each repetition. Each direct Monte Carlo estimate  uses the same number of function evaluations and sample sizes as PPSVMG.  We show an estimate box plot in Figure~\ref{fig:lotka_estimation}.}
\begin{figure}[htbp]
    \centering
    \includegraphics[width=.8\linewidth]{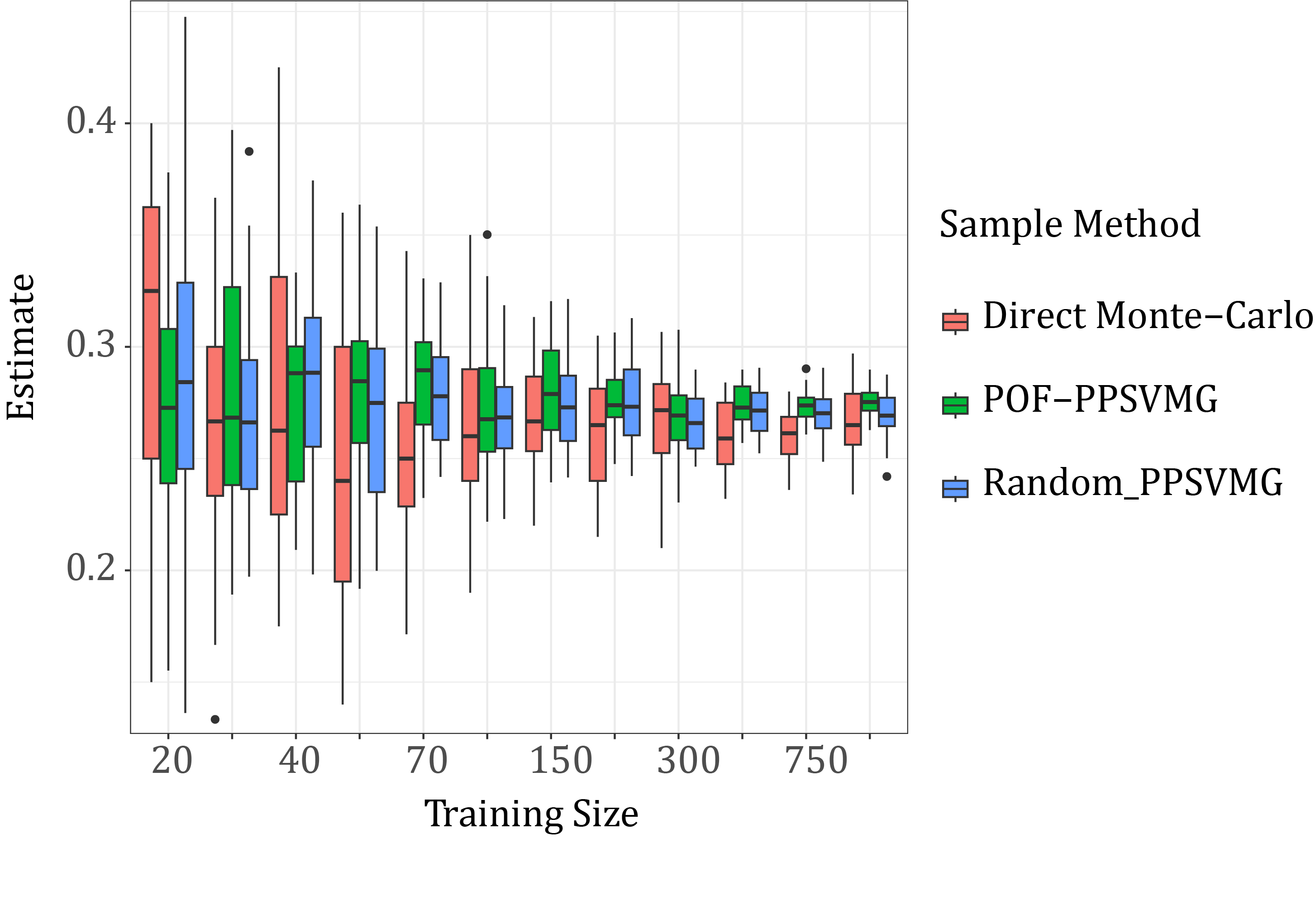}
    \caption{{Estimate box plot for $P(z_3(10)>0.8)$}}
    \label{fig:lotka_estimation}
\end{figure}

{We observe a slight overestimation in the PPSVMG results compared to the direct Monte Carlo approach.  Welch’s t--test for equality of means, see Figure~\ref{fig:lotka_estimation_pvalue}, provides evidence of a difference in estimated means between the two methods.  However, all PPSVMG estimated means are within the confidence interval of the direct Monte Carlo method. We also observe significant evidence that the PPSVMG estimate has significantly lower variance than the direct Monte Carlo estimate.}

\section{Conclusion}\label{sec:conclu}

\begin{figure}[htbp]
	\centering
	\includegraphics[width=\linewidth]{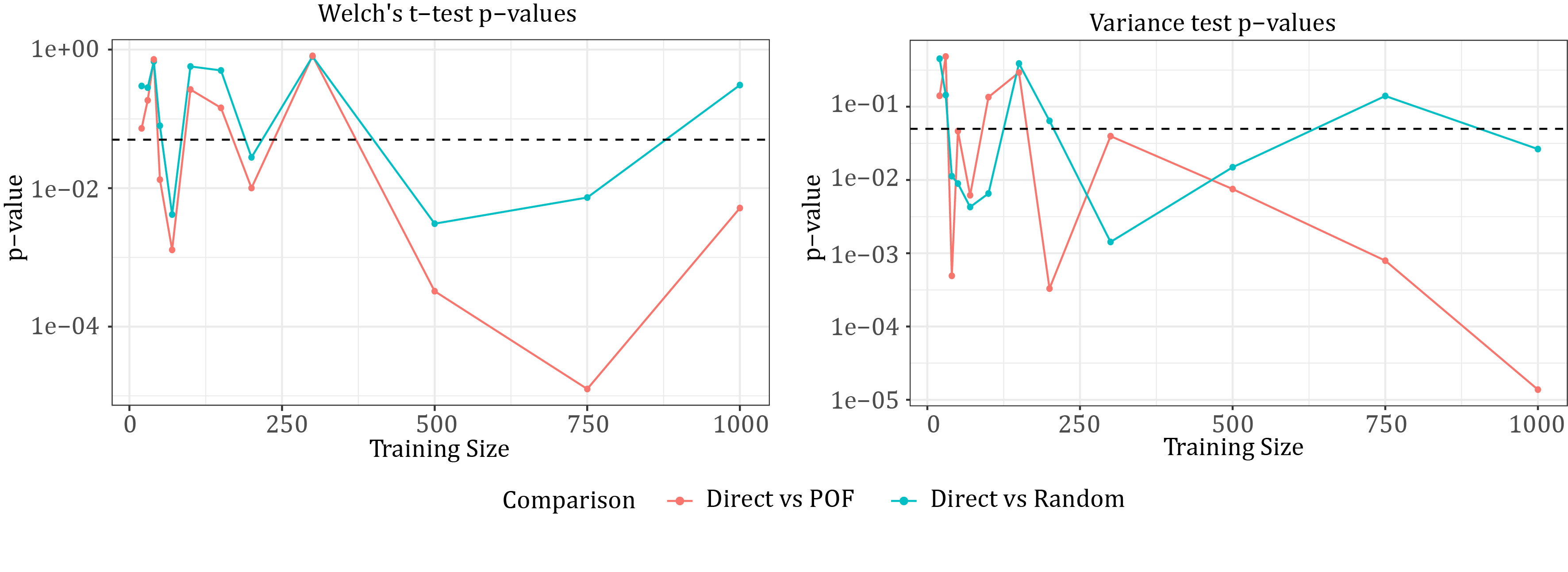}
	\caption{{Equal mean and one--sided variance hypothesis tests comparing direct Monte Carlo to  PPSVMG based on POF--Darts sampling. Top: Welch’s t--test for equality of means. Bottom: One--sided variance test with the alternative hypothesis that the variance of the direct method is largest.}}
	\label{fig:lotka_estimation_pvalue}
\end{figure}

{We introduce a novel machine learning methodology called The \emph{Penalized Profile Support Vector Machine based on the Gabriel edited set} for computation of the probability of failure of a complex system as determined by a nonlinear decision manifold defined by the inverse of the solution of a process model of the system.}
Formulating the computation as a classification problem, the method  is designed to minimize the number of evaluations of the model while preserving geometric integrity of the decision boundary to increase interpretability of the surrogate classification boundary while preserving robust accuracy.

The method employs POF--Darts sampling  to strategically allocate sample points near the decision boundary. It builds an approximate decision boundary consisting of linear SVMs based on clusters of Gabriel neighbors. The linear SVMs are computed using a penalty term designed to preserve the geometry of the true boundary. Classification of points are determined by an ensemble average with weights  computed using distances between the point being classified and the clusters. 
We prove two basic convergence results that together imply that the local linear SVM models become closer to the nonlinear decision boundary as the  sample size increases. In practice, the method achieves robust accuracy comparable to  other state of the art classification methods.

\section*{Acknowledgments}

Donald Estep acknowledges the support of the Natural Sciences and Engineering Research Council of Canada and the Canada Research Chairs program. Jacob Zhu acknowledges the support of the Natural Sciences and Engineering Research Council of Canada. 

\bibliographystyle{abbrv}
\bibliography{interactapasample}

\appendix

\section{Details  for ingredients}\label{sec:details}

\subsection{POF--Darts}\label{det:POFDarts}
POF--Darts determines the sample points $\{x_i \}_{i=1}^n\subset \Lambda$   in a sequential fashion. Given $\{x_i\}_{i=1}^{m-1}$, a new sample $x_m$ is associated with a sphere $S_m$ centered at $x_m$ with initial radius,
\begin{equation}\label{POFinirad}
	r(x_m) = \frac{|Q(x_m) - q_0|}{\mathcal{L}\, L_m}, 
\end{equation}
where $L_m$ is an estimate of the maximum magnitude of the gradient of $Q$ in a region containing the sample point and the nearby portion of the nonlinear decision boundary and $\mathcal{L}$ is a parameter that reduces the chances of underestimation of the distance. Intuitively, (\ref{POFinirad}) is suggested by using a first order Taylor expansion of $Q$ around $x_m$ to obtain an estimate of the distance between $x_m$ and a point on the nonlinear decision boundary. The maximum magnitude of the gradient of $Q$ is obtained by evaluating the gradient of $Q$ at neighboring samples. The simplest choice is $L_m=|\nabla Q(x_m)|$.

The radius of a sphere is  adjusted in the case of overlap with previously determined spheres. Overlap of spheres of points within the same classification region is permitted. But, overlap between spheres associated with points in different classification regions is not permitted. In this case, the  radii are adjusted as follows. If $x_i$ and $x_j$ have different classifications and their spheres overlap, we set
\[
r(x_i) = \frac{|Q(x_i) - q_0|}{L},\quad r(x_j) = \frac{|Q(x_j) - q_0|}{L},
\]
where $L = \frac{|Q(x_i) - Q(x_j)|}{\|x_i - x_j\|}$. 

The new sample $x_m$ is added in the complement of $\{S_i\}_{i = 1}^{m-1}$, i.e., 
\[
x_m \in \Lambda \bigg\backslash \bigcup_{i = 1}^{m-1} S_i.
\]
We use accept--reject random sampling to achieve this.  We present the algorithm in Algorithm~\ref{alg:POF} in Appendix~\ref{Algs}.

Unfortunately, the region not covered by previous spheres becomes an increasingly small portion of the total domain, making it increasingly challenging to determine new samples because of an unrealistic fail rate in accept--reject sampling.  To deal with this, we adapt the POF--Darts with $1-d$ darts sampling technique  originally introduced in (\cite{Ebeida_Kd_Darts}).  A $1-d$ dart is a $1$-dimensional hyperplane through a specified point, known as a $1-d$ flat, that is determined by letting $1$ coordinate vary while fixing the remaining $d-1$ coordinates of the point. We present the Algorithm~\ref{alg:1ddart} in Appendix~\ref{Algs}.

$1-d$ POF--Darts is costlier in terms of constructing individual samples than straightforward random sampling but exhibits a substantially lower rejection rate than accept--reject sampling when the region of exploration is small. The POF--Darts algorithm converges at an exponential rate once the density of sample points in the domain is sufficiently dense (\cite{Ebeida_Kd_Darts}).

\subsection{The Voronoi edited set}\label{det:GES}
One approach to determining points closest to a nonlinear decision boundary is based on the \textit{Voronoi diagram} for a collection of points, which is a partition of the region into polygons such that each point lies in a distinct polygon. Two points are \textit{Voronoi neighbors} if their polygons share a common boundary. The \textit{Voronoi edited set (VES)} is the subset of  points whose Voronoi neighbors have different classifications (\cite{BinaryK}).  The VES can be used to define an approximate decision boundary using the nearest neighbor rule that is consistent in the sense that it matches the decision boundary constructed using the full sample set. Unfortunately, the computational complexity of the algorithm determining the VES  grows exponentially with dimension.

The GES does not have the property of a consistent decision boundary of the Voronoi edited set, but it offers a feasible computational complexity. An efficient algorithm proposed in (\cite{BinaryK}) has complexity of $O(dn^2)$ on average while performance of classification via nearest neighbor using the GES closely matches those relying on the Voronoi edited set. 

\section{Other classification approaches}\label{App:othermethod}

In this Appendix, we briefly describe other classification approaches used for benchmark purposes.

\subsection{Pujol's method}\label{det:Pujol}

Pujol's method attempts to deal with the lack of decision boundary consistency by fine tuning the model classification constructed using the GES by checking the predictions against the full training set. The predicted classification of a point $x\in\Lambda$ is computed according to the sign of the  projection,
\[
\sigma_{ij}(x)  = \text{sign} \left((x - \overline{x}_{ij})\, \cdot\, \frac{(x_i - x_j)}{\|x_i - x_j\|} \right), 
\]
calculated for pair of Gabriel neighbors. This is equivalent to defining linear classification boundaries consisting of the hyperplanes that passes through each CBP  perpendicular to the Gabriel graph associated with the CBP, see Fig.~\ref{fig:baseclassifier}.
\begin{figure}[htb]
	\centering
	\includegraphics[width=0.7\textwidth]{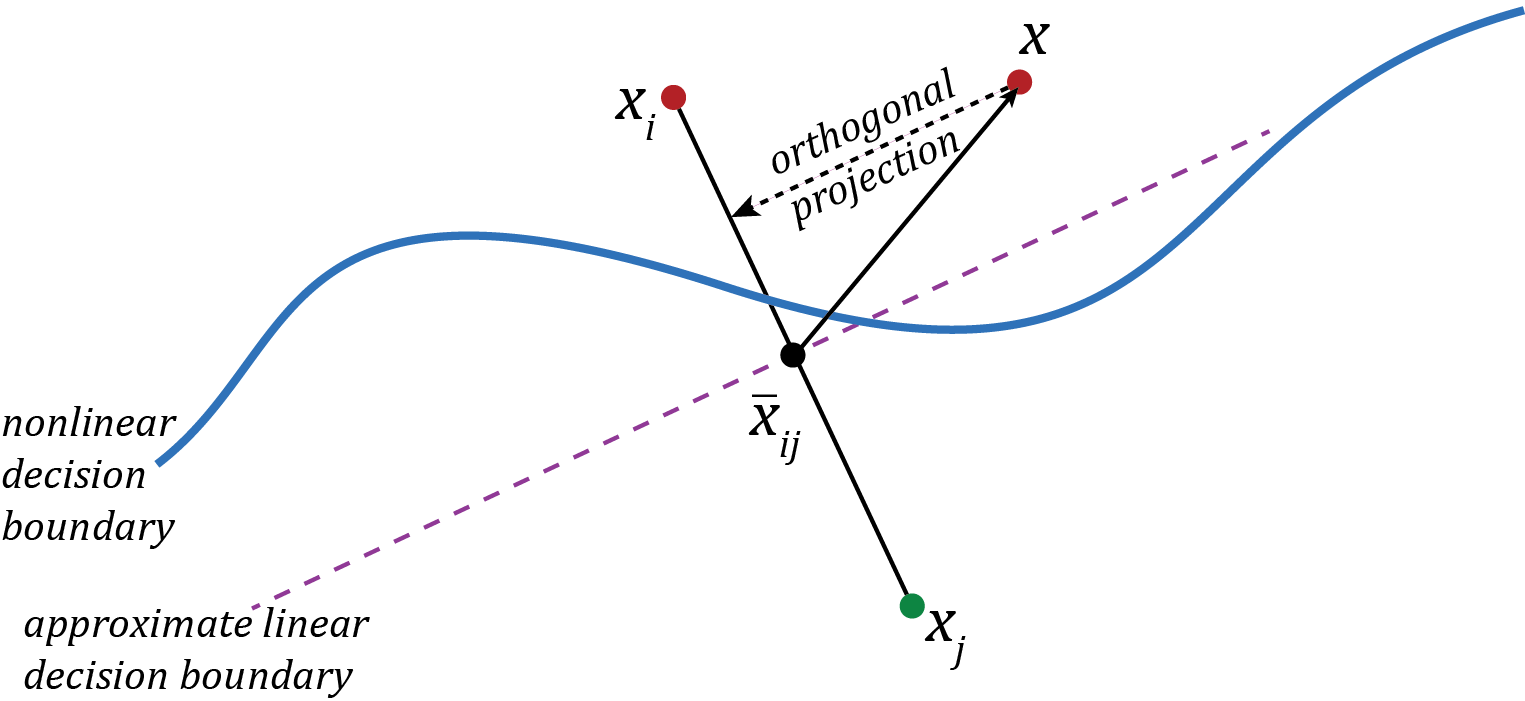}
	%    \centering
	%    \begin{tikzpicture}[line width=1.1pt, scale = 0.6]
		%    %\draw[black] (5,5) circle (2);
		%    \filldraw[black, dotted] (5,5) circle (2pt) node[anchor=south east]{$x_{cp}$};
		%    \filldraw[black] (5,3) circle (2pt) node[anchor=west]{$x_i^{''}$};
		%    \filldraw[red] (5,7) circle (2pt) node[anchor=west]{$x_i^{'}$};
		%    \draw[black,dotted]    (5,3) to (5,7);
		%    \draw[->, black]    (5,5) to (5,7);
		%    \draw[->, black]    (5,5) to (7,6);
		%    \filldraw[red] (7,6) circle (2pt) node[anchor=west]{$x_j$};
		%    \draw[blue, dotted]    (1,5) to (10,5);
		%    \node[] at (10,4.5) {Decision boundary};
		%    \end{tikzpicture}
	\caption{The implicit linear classification boundary defined in Pujol's method}
	\label{fig:baseclassifier}
\end{figure}

The ensemble classifier for the $\ell^{th}$ training point $x_\ell$, $1 \leq \ell \leq n$, is  defined,
\[
\pi(x_\ell) = \sum_{(i,j)\in\mathcal{N}} \sigma_{ij}(x_\ell)\omega_{ij},
\]
where $\pi(x_\ell) \in [-1, 1]$ and $\{\omega_{ij}\}_{(i,j)\in\mathcal{N}}$ are weights chosen to minimize the fitted sum of squares,
\[
\argmin_{\omega} \left(\sum_{\ell=1}^N|\pi(x_\ell) - y_\ell|^2 + \lambda^2\|\omega - \omega^{*}\|^2 \right),
\]
where $\omega^*$ is an initial guess, e.g., $\omega^{*}_{ij} = \frac{1}{|\mathcal{N}|}$ for $(i,j)\in\mathcal{N}$ and $|\mathcal{N}|$ is the number of points in $\mathcal{N}$, and $\lambda$ is a regularization parameter. An explicit formula for the solution can be determined by calculus as usual.

Pujol, et al argue that their method exhibits comparable performance to kernel SVM while surpassing KNN and linear SVM on most datasets. Pujol's method does not take into account knowledge of $Q$ or the geometry of the nonlinear decision boundary beyond what can be gleaned from the values on a given training set. Consequently, the computed classification boundary may be very misaligned with the true decision boundary and Pujol's method does not provide an intuitive geometric interpretation. We study performance in Sec.~\ref{sec:simul}.

\subsection{KNN-SVM}\label{App:KNN-SVM}
Zhang, et al (\cite{zhang2006svm}) introduce  the K Nearest Neighbor SVM (KNN-SVM) method, which constructs a localized linear SVM model for each sample point in the training set using a nearest neighbors approach to determine a subset of training points for the model. This approach frequently results in notable performance improvements compared to nonlinear SVMs while circumventing the challenging task of model selection. 

\subsection{PSVM  and MagKmeans clustering}\label{App:PSVM}
The cost of computing a localized linear SVM model at every training point is prohibitive. Cheng,  et al (\cite{Chen2010}) proposed the Profile SVM (PSVM) method that partitions the training dataset into  clusters and computes a  linear SVM model for cluster. Their approach uses the unsupervised clustering MagKmeans method as an attempt to  balance distribution of points with different classifications within each cluster. 

This begins by applying Lloyd's k-means clustering to the training set. Now, we let $Z$ denote an $n \times k$ matrix where $n$ is the size of the training set $\{\mathcal{X},\mathcal{Y}\}$ and $k$ is the number of clusters. The $i^{th}$ row of $Z$ indicates the membership of $x_i$, $1\leq i\leq n$. If $x_i$ is part of the $j^{th}$ cluster then $Z_{i,j}=1$ otherwise  $Z_{i,j}=0$. Lloyd's k-means clustering algorithm starts by randomly picking $k$ points in $\mathcal{X}$ as cluster centroids $\mathcal{C} = \{C_i\}_{i=1}^k$, then assigns the points to each cluster by minimizing the intra-cluster distance,
\begin{equation*}
	\argmin_{Z} \sum_{i=1}^{n}\sum_{j=1}^{k} Z_{i,j}\,\|x_i - C_j\|^2.
\end{equation*}
Following this, the centroids are recomputed and new clusters are computed. This iterated  until convergence to a steady state.

The MagKmeans clustering algorithm adds an additional penalty term to the objective function,
\begin{equation}
	\label{eq:MagKmeans}
	\argmin_{Z} \sum_{i=1}^{n}\sum_{j=1}^{k} Z_{i,j}\|x_i - C_j\|^2 + \gamma\sum_{j=1}^{k}\left|\sum_{i=1}^{n} Z_{i,j}y_i\right|,
\end{equation}
where $\gamma$ is a nonnegative penalization parameter. The penalty term quantifies the asymmetry in class distributions within each cluster and by minimizing this component, the algorithm strives to achieve a balanced representation of classifications within each cluster. 

Unfortunately, solving (\ref{eq:MagKmeans}) with the cluster membership matrix set to a binary classification is NP Hard. Cheng, et al (\cite{Chen2010}) relax the condition by allowing $Z$ to be real valued in $ [0,1]$ and defining the objective function to be,
\begin{align}
	\label{eq:MagKmeans2}
	\begin{split}
		&\argmin_{Z} \sum_{i=1}^{n}\sum_{j=1}^{k} Z_{i,j}\|x_i - C_j\|^2 +  \gamma \sum_{j=1}^{k}t_j,\\
		&\text{such that} \quad -t_j \le \bigg|\sum_{\ell=1}^{n} Z_{\ell,j}y_i \bigg| \le t_j,\; t_j \ge0,\; 0\le Z_{i,j} \le 1, \;\sum_{i=1}^{k} Z_{i,j} = 1,\, \forall i,j.
	\end{split}
\end{align}
We solve (\ref{eq:MagKmeans2}), using the ECOS conic solver and Gurobi optimizer in the cvxpy package in Python (\cite{diamond2016cvxpy,domahidi2013ecos,gurobi}). We present the algorithm in Appendix~\ref{Algs}.

The PSVM with MagKmeans clustering  has several flaws. First, the partition of the cluster does not consider the geometry of the decision boundary. As depicted in Fig.~\ref{fig:PSVM_comparison} (c), points in clusters 0 and 3 are not linearly separable, even though the intracluster distribution of different classes is balanced. Also, relaxing the optimization conditions to make it solvable in polynomial time results in suboptimal clustering and solutions. Moreover, the final number of clusters is not controllable. In Fig.~\ref{fig:PSVM_comparison}, PSVM ends up with considerably fewer clusters than the initial setting. Finally, effectiveness of PSVM is greatly influenced by the distribution of the training dataset.

\subsection{Operational details for the classification methods}\label{app:hyper}
The codes and parameter values used for classification methods used in the numerical experiments are specified at \url{https://github.com/jacobz0106/classificationGE.git}
We list the software in Table~\ref{tab:methods_defination}.

\begin{table}[htb]
	\centering
	\begin{tabular}{ |p{6cm}||p{7cm}|  }
		\hline
		\multicolumn{2}{|c|}{Benchmark classification software} \\
		\hline
		Method & Python packages\\
		\hline
		K nearest neighbors (KNN) & sklearn.neighbors.KNeighborsClassifier \\
		Random forest (rf)   & sklearn.ensemble.RandomForestClassifier  \\
		XGBoost &   xgboost.xgb  \\
		multilayer perception & sklearn.neural$\_$network.MLPClassifier  \\
		kernel SVM (SVM$\_$kernel)    & sklearn.svm.svc \\
		\hline
	\end{tabular}
	\caption{Classification methods and related Python packages used in the comparison}
	\label{tab:methods_defination}
\end{table}

\section{Details for computation shown in Figures \ref{fig:GPSVM_comparison} and \ref{fig:PSVM_comparison} }\label{App:PSVM_comp}

\begin{enumerate}
	\item Classification in $\Lambda$ by a linear boundary: \[Q = -1 \text{ if } x_1\le 5 \text{ else } Q=1\]
	\item Classification in $\Lambda$ by an oscillatory boundary: 
	\[
	Q = 1 \text{ if } (x_1 \le 2 \text{ or } x_1 \ge 8 \text{ or } x_2 \ge 8) | (x_1 \ge 4 \text{ and } x_1 \ge 6 \text{ and } x_2 \ge 2) \text{ else } Q=-1
	\]
	\item Partition of $\Lambda$ by a spiral boundary defined,
	\begin{gather*}
	x_1 = \pm \left(1 + \theta\right)^a \cos(\theta) + \left(b_1 - 0.5\right) \frac{c}{2},\\
	x_2 = \pm \left(1 + \theta\right)^a \sin(\theta) - 0.75 + \left(b_2 - 0.5\right) \frac{c}{2},
	\end{gather*}
	with $c\geq 1$, $0< a \leq 1$, $0 \leq b_1 \leq 1$, and $0\leq \theta \leq 2 \pi$.
	
\end{enumerate}

\section{Numerical solution of differential equation models}\label{app:nummethod}

\subsection*{Brusselator model}
We solve the Brusselator model using numerical method with fixed time step $h = T/N$ and $N$ is the number of discretization points. We use $N=100$ in the computations in this paper. The solution is smooth and there are no difficulties computing accurate numerical approximations. 

\subsection*{Elliptic model}
We solve the elliptic model using a standard finite difference approximation with $\frac{du}{dx} \approx \frac{u(x) - u(x - h) }{h}$ and $\frac{d^2 u }{d x^2} \approx \frac{ u(x + h) - 2u(x) + u(x - h) }{h^2}$ where the step size $h = 1/N$ and $N$ is the number of discretization points. We use $N=100$ in the computations in this paper. The resulting discrete system is solved using Gaussian elimination.

\section{Algorithms}\label{Algs}

%1
\begin{algorithm}
	%\ContinuedFloat
	\caption{POF--Darts Algorithm}
	\label{alg:POF}
	\textbf{Input:} the number of initial points $n$, the size of the output sample set $m$, the parameter domain $\Lambda$, the critical value $q_0$,  the mapping function $Q$
	
	\textbf{Output:} a sample data set containing $m$ rows $\{x_i,Q_i,\nabla Q_i\}_{i=1}^m$
	\medskip
	
	\SetKwBlock{Beginn}{beginn}{ende}
	\Begin{
		\textbf{Initialization:} Randomly select $n$ points $\{x_i\}_{i=1}^n \in \Lambda$;
		
		Evaluate $Q(x_i)$, $i = 1, \dots, n $;
		
		Classify points by comparing $Q(x_i)$ with $q_0$ for $i = 1, \dots, n$;
		
		Evaluate the gradients $\nabla Q(x_i)$, $i = 1, \dots, n$;
		
		Calculate the radii of the point $x_i$, estimated by
		\begin{equation}
			r_i = \frac{|Q(x_i) - q_0|}{const * \max\{\text{threshhold} ,|Q'(x_i)|\} }, i = 1,\dots,n.  \label{eq:radius}
		\end{equation}

		\textit{When the gradient magnitude is small, the function mapping becomes inefficient. To prevent division by zero or extremely small values, we establish a minimum threshold for the gradient.}
		\medskip
		
		\For{$j = 1$ to $m - n$}{
			\If{the spheres centered at $x_p$ and $x_q$ overlap where $x_p$ and $x_q$ have different classifications}{
				remove the overlap by adjusting the radii:
				\[
				r_p = \frac{|Q(p) - q_0|}{|Q(x_p) - Q(x_q)|}*d(x_p, x_q), \quad r_q = \frac{|Q(q) - q_0|}{|Q(x_p) - Q(x_q)|}*d(x_p, x_q)
				\]
			}
			\EndIf{}
			Locate a point $x_{new}$ outside pre-existing prior spheres, using the $k-d$ darts method;
			
			Calculate its radius by equation~\ref{eq:radius};
			
			Evaluate $Q$ at $x_{new}$;
			
			Find the classification by comparing $Q(x_{new})$ with $q_0$;
			
			Evaluate the gradient $\nabla Q(x_{new})$. 
		}
		\Return the sample set containing the $m$ points.
	}% end for begin
	
\end{algorithm}

%2
\begin{algorithm}
	\caption{1-d Dart Algorithm~(\cite{Ebeida_Kd_Darts})}
	\label{alg:1ddart}
	
	\textbf{Input:} a collection of preexisting prior disks, sampling space $\Omega \in \mathbb{R}^d$ 
	
	\textbf{Output:} a d-dimensional point that is not overlapped with preexisting prior disks or a miss.
	\medskip
	
	\SetKwBlock{Beginn}{beginn}{ende}
	
	\While{number of misses havn't reached a specific number}{
		initiate a 1-d dart $t^1 $ by randomly sampling a $d$-dimensional point in $\Omega $\;
		\For{all $j =  1$ to $d$}{
			define a $1$-dimensional flat (line), denoted as $t^1_j$, by allowing the $j^{th}$ coordinate to vary\;
			construct a line segments $g = t^1_j \bigcap \Omega $\;
			\For{every preexisting priori disk $ c $}{
				removing the overlap parts of disk $c$ and $g$ by $ g = g \setminus D(c) $\;
			}
			\If{$g \neq \varnothing $}{
				mark $t^1$ as a hit\;
				draw a $1$-d sample $p$ in $j^{th}$ coordinate uniformly along $g $\;
				\Return the $d$-d sample by fixing $j^{th}$ coordinate of $t^1_j$ as $p$\;
			}
		}
		mark $t^1$ as a miss\;
	}
	
\end{algorithm}

%3
\begin{algorithm}
	\caption{Improved Gabriel Edited algorithm}\label{alg:improvedGEA}
	
	\textbf{Input}: a set of points $\{\mathcal{X},\mathcal{Y}\} =\{x_i,y_i\}, x_i\in R^d$ belonging to class $y_i = \{-1,1\}$.
	
	\textbf{Output}: set of pair of indexes $\{(i,j)\}$ corresponding to the data pair that defines the CBP.
	
	\medskip
	
	\SetKwBlock{Beginn}{beginn}{ende}
	
	\Begin{
		
		Create an empty set $U$;
		
		\For{each $x_i|y_i = 1$ and $x_i \in \mathcal{X}$}{
			
			Create a set $N = \{ x_j|y_j = -1, x_j\in \mathcal{X}\}$;
			
			\For{\texttt{each} $x_j|y_j = -1$ \texttt{and} $x_j \in \mathcal{X}$}{
				compute  the midpoint $x_m$ of $x_i$ and $x_j$;
				\For{\texttt{each} $x_k \in \mathcal{X}$ \texttt{and} $x_k \ne x_i \ne x_j$} {
					\uIf{$d_{km} \le \frac{d_{ij}}{2}$ \tcc*{ $d_{km}$: Euclidean distance between $x_i$, $x_m$; }} { 
						remove $x_j$ from $N$;
						
						continue;
						
					}
					\uElseIf{$x_k \in N$}{
						compute  the midpoint $x_m^{'}$ of $x_i$ and $x_k$;
						
						\If{$d_{km^{'}} \le \frac{d_{ik}}{2}$} {
							remove $x_k$ from $N$;
						}
					}      
				}
			}
			\texttt{add $\{(i,j) | x_j \in N\}$ to $U$.}
		}
		
	} %end for begin
	
\end{algorithm}

%4
\begin{algorithm}
	\caption{PPSVMG clustering algorithm based on CBPs}\label{alg:PPSVMG_clustering}
	
	\textbf{Input}: Cluster size $K$ and similarity $s$, training set $\{\mathcal{X}, \mathcal{Y}\}$ $=$ $\{x_i,y_i\}_{i=1}^n$, Gabriel neighbor pairs $\mathcal{G} =\{(x_i,x_j)\}_{(i,j)\in \mathcal{N}}$, and indices $\mathcal{N}$
	
	\textbf{Output}: A set of clusters containing the Gabriel neighbors $C =\{C_i\}_{i=1}^{M}$.
	
	\medskip
	
	\SetKwBlock{Beginn}{beginn}{ende}

	\Begin{
		\textbf{Initial clusters:}
		Find the CBPs by $\overline{\mathcal{G}}=\{\overline{x}_{ij}\}_{(i,j) \in {\mathcal{N}}}$, where $\overline{x}_{ij} =\frac{1}{2}\big(x_i + x_j \big)$ for $(i,j)\in\mathcal{N}$;
		Let C = \{\};
		
		\For{each $(i,j) \in \mathcal{N}$}{
			Find the $K$ nearest neighbors of $\overline{x}_{ij}$ in $\overline{\mathcal{G}}$, denote as $\{\overline{x}_{uv}\}_{(u,v)\in Neigh(i,j)}$\tcc*{Neigh(i,j): Index set for $K$ nearest neighbors of $x_{ij}$; }

			Add the cluster $\{x_u,x_v\}_{(u,v)\in Neig(i,j)}$ to C;
		}
		Compute the cluster similarities $\{S_{ij}\}_{i,j = 1}^{|\mathcal{G}|}$, where $S_{ij} = \frac{ |C_i \bigcap C_j| }{k}$.
		
		\textbf{Remove similar clusters:}
		
		\For{each $i$ in $I = \{j\}_{j = 1}^{|\mathcal{G}|}$ }{
			Set $C_j = \emptyset$ and remove $j$ from $I$ for every $S_{ij} \ge s, i\ne j;$  
		}
		\Return{C with nonempty clusters.}

	}
\end{algorithm}

%5
\begin{algorithm}
	\caption{Penalized support vector machine}
	\label{alg:penal_SVM}
	\medskip
	\textbf{Input}: a set of points $\{\mathcal{X},\mathcal{Y}\} =\{x_i,y_i\}, x_i\in R^d$ belonging to class $y_i = \{-1,1\}$, values of the gradient $\{\nabla Q\}$ at the training points, a testing set $\{x^{'}_j\}_{j=1}^{n}$, slack variable $C$ and penalty constant $\kappa$. 
	
	\textbf{Output}: Predicted classifications for $\{x^{'}_j\}$.
	
	\medskip
	
	\SetKwBlock{Beginn}{beginn}{ende}
	
	\textbf{Training}: 
	
	\Begin{
		Find the norm vector $\mathbf{w^{*}}$ of the separating plane of soft--margin linear SVM by optimizing (\ref{softSVM});
		
		Solve $\alpha$ by optimizing the objective function defined by (\ref{optimization});
		
		Calculate $t^*$ by (\ref{eqn_t}) and $w(t) = t^{*}w^* + (1 - t^{*})w_{Q};$
		
		Calculate $b =  \frac{\sum_{\alpha_j > 0 } y_j - w(t)^Tx_j}{|\{\alpha_j|\alpha_j > 0\}|}$
	}
	\textbf{Predicting:}
	
	\Begin{
		return 1 if $w(t)^Tx_j^{'} + b >=0$ else -1.
	}
\end{algorithm}

\begin{algorithm}
	\caption{PSVMG algorithm}\label{alg:G-PSVM}
	
	\textbf{Input}:  a set of points $\{\mathcal{X},\mathcal{Y}\} =\{x_i,y_i\}, x_i\in R^d$ belonging to class $y_i = \{-1,1\}$, a testing set $\{x^{'}_j\}_{j=1}^{n}$ and number of initial clusters $k$, and $M\leq k$ for the number of clusters to use in the weighted prediction.
	
	\textbf{Output}: Predicted classifications for $\{x^{'}_j\}_{j=1}^{n}$.
	
	\medskip
	
	\SetKwBlock{Beginn}{beginn}{ende}
	
	\textbf{Training}:
	
	\Begin{
		
		Find the GES of $\{\mathcal{X},\mathcal{Y}\}$;
		
		Find the Characteristic Boundary Points (CBPs);
		
		Partition the CBPs into $k$ clusters by Lloyd's k-means algorithm;
		
		Construct $k$ new clusters by taking the corresponding Gabriel neighbors of the CBPs in each cluster, denote the cluster centroids as $\{C_i\}_{i=1}^k$;
		
		\For{each cluster}{
			Create a linear SVM on that cluster;
		}    
		
	}
	\textbf{Prediction}:
	
	\Begin{
		
		\For{each $x^{'}_j$}{
			Find the $M$ nearest clusters to $x^{'}_j$ by the distance to the centers of the clusters $\{C_i\}_{i=1}^k$, record the distances $\{d_i\}_{i=1}^M$;
			
			\For{each SVM related to the neighboring clusters}{
				predict $\hat{l}_i$ for $x^{'}_j$;
			}
			
			Final prediction for $x^{'}_j$ is based on the ensemble:\[
			\hat{l} = \sum_{i=1}^{M} \frac{\hat{l_i}}{d_i\sum_{j=1}^{M} 1/d_j}\]
			
		}

	}
\end{algorithm}

\begin{algorithm}
	\caption{PPSVMG algorithm}\label{alg:G-PSVM}
	
	\textbf{Input}:  a set of training points $\{\mathcal{X},\mathcal{Y}\} =\{x_i,y_i\}, x_i\in R^d$ belonging to class $y_i = \{-1,1\}$, a prediction set $\{x^{'}_j\}_{j=1}^{n}$ and initial cluster size $K$, and $M$ for the number of clusters to use in the weighted prediction.
	
	\textbf{Output}: Predicted classifications for $\{x^{'}_j\}_{j=1}^{n}$.
	
	\medskip
	
	\SetKwBlock{Beginn}{beginn}{ende}
	
	\textbf{Training}:
	
	\Begin{
		
		Find the GES of $\{\mathcal{X},\mathcal{Y}\}$ by Algorithm~\ref{alg:improvedGEA};
		
		Find the Characteristic Boundary Points (CBPs) by taking the midpoints of the pairs of Gabriel neighbors in the GES;
		
		Produce the set of clusters $C$ of GES;

		\For{each cluster in $C$}{
			Construct a penalized SVM;
		}    
		
	}
	\textbf{Prediction}:
	
	\Begin{
		
		\For{each $x^{'}_j$}{
			Find  $M$ nearest clusters to $x^{'}_j$ by the distance to the centers of the clusters $\{C_i\}_{i=1}^K$, record the distances $\{d_i\}_{i=1}^M$;
			
			\For{each SVM related to the neighboring clusters}{
				predict $\hat{l}_i$ for $x^{'}_j$;
			}
			
			Final prediction for $x^{'}_j$ is based on the ensemble:\[
			\hat{l} = \sum_{i=1}^{M} \frac{\hat{l_i}}{d_i\sum_{i=1}^{M} 1/d_i}\]
			
		}

	}
\end{algorithm}

\begin{algorithm}
	\caption{MagKmeans Clustering Algorithm}
	\label{alg:MagKmeans2}
	
	\textbf{Input}:  a set of points $\{\mathcal{X},\mathcal{Y}\} =\{x_i,y_i\}, x_i\in R^d$ belonging to class $y_i = \{-1,1\}$,  number of clusters $k$.
	
	\textbf{Output}: Cluster classifications for each $X_i$.
	
	\medskip
	
	\SetKwBlock{Beginn}{beginn}{ende}
	\SetKwRepeat{Do}{do}{while}%
	\Begin{ 
		Initialize $k$ centroids $C = \{C_i\}_{i=1}^k$ by randomly select $k$ points from $\{x_i\}_{i=1}^{n}$;

		\For{$i \in \{1$ to the maximum number of iterations$\}$}{
			Solve $Z$  in equation~\ref{eq:MagKmeans2} by linear programming;
			
			Update the new centroids $C^{'}$ by:\[
			C^{'}_j = \frac{\sum Z_{i,j}X_i}{\sum_{i = 1}^{n} Z_{i,j}
			},
			\]
			\If{$C^{'} = C$}{
				terminate and return $Z$.
			}\Else{
				$C = C^{'}$.
			} 
		}
		Reached maximum iteration and failed to obtain a stationary result. Repeat the above by choosing another random set of initial centroids. 
	}
\end{algorithm}

\end{document}